\documentclass[10pt, a4paper, logo, copyright]{deepmind}
\renewcommand{\today}{}

\pdftrailerid{redacted}

\makeatletter
\renewcommand\bibentry[1]{\nocite{#1}{\frenchspacing\@nameuse{BR@r@#1\@extra@b@citeb}}}
\makeatother

\usepackage{kantlipsum, lipsum}
\usepackage{dsfont}
\usepackage[authoryear, sort&compress, round]{natbib} 
\usepackage[boxruled,vlined,noend]{algorithm2e}

\usepackage[capitalize,noabbrev]{cleveref}
\usepackage{footmisc}
\usepackage{microtype}
\usepackage{graphicx}
\usepackage{subfigure}
\usepackage{booktabs} 
\usepackage{amsmath}
\usepackage{amssymb}
\usepackage{mathtools}
\usepackage{amsthm}
\usepackage{float}
\usepackage{xcolor}
\usepackage{multirow}
\usepackage{dirtytalk}
\usepackage{nicefrac}
\usepackage{xurl}
\usepackage[textsize=tiny]{todonotes}
\usepackage{wrapfig}
\usepackage{pifont}
\usepackage{bbding}

\usepackage{diagbox}

\usepackage{xcolor}

\usepackage{thm-restate}
\newtheoremstyle{theoremdd}
  {\topsep}
  {\topsep}
  {\itshape}
  {0pt}
  {\bfseries}
  {. }
  { }
  {\thmname{#1}\thmnumber{ #2}\textnormal{\thmnote{ (#3)}}}
\theoremstyle{theoremdd}

\DeclareMathOperator{\diag}{diag}

\usepackage{nicematrix}

\newcommand\R{\mathbb{R}}
\newcommand\Cmp{\mathbb{C}}

\newcommand\E{\mathbb{E}}

\newcommand\Tstrut{\rule{0pt}{2.6ex}}       
\newcommand\Bstrut{\rule[-0.9ex]{0pt}{0pt}} 
\newcommand{\TBstrut}{\Tstrut\Bstrut} 

\usepackage{amsthm}
\usepackage{enumitem}

\theoremstyle{plain}

\theoremstyle{definition}

\theoremstyle{remark}

\title{Resurrecting Recurrent Neural Networks for Long Sequences}

\correspondingauthor{antonio.orvieto@inf.ethz.ch, sohamde@deepmind.com}

\renewcommand{\today}{2023-03-11} 

\author[1,+]{Antonio Orvieto}
\author[2]{Samuel L Smith}
\author[2]{Albert Gu}
\author[2]{Anushan Fernando}
\author[2]{Caglar Gulcehre}
\author[2]{Razvan Pascanu}
\author[2]{Soham De}

\affil[1]{ETH Zurich}
\affil[2]{DeepMind}
\affil[+]{Work done at DeepMind.}

\begin{abstract}
\vspace{-3mm}
Recurrent Neural Networks (RNNs) offer fast inference on long sequences but are hard to optimize and slow to train. Deep state-space models (SSMs) have recently been shown to perform remarkably well on long sequence modeling tasks, and have the added benefits of fast parallelizable training and RNN-like fast inference. However, while SSMs are superficially similar to RNNs, there are important differences that make it unclear where their performance boost over RNNs comes from. 
In this paper, we show that careful design of deep RNNs using standard signal propagation arguments can recover the impressive performance of deep SSMs on long-range reasoning tasks, while also matching their training speed. 
To achieve this, we analyze and ablate a series of changes to standard RNNs including linearizing and diagonalizing the recurrence, using better parameterizations and initializations, and ensuring proper normalization of the forward pass. Our results provide new insights on the origins of the impressive performance of deep SSMs, while also introducing an RNN block called the Linear Recurrent Unit that matches both their performance on the \emph{Long Range Arena} benchmark and their computational efficiency.
\end{abstract}

\begin{document}

\maketitle

\section{Introduction}

Recurrent neural networks (RNNs) have played a central role since the early days of deep learning, and are a natural choice when modelling sequential data~\citep{mcculloch1943logical,hopfield1982neural,rumelhart1985learning,elman1990finding}. 
However, while these networks have strong theoretical properties, such as Turing completeness~\citep{kilian1996dynamic,chung2021turing}, it is well-known that they can be hard to train in practice. In particular, RNNs suffer from the vanishing and exploding gradient problem~\citep{hochreiter1991untersuchungen,bengio1994learning, pascanu2013difficulty}, which makes it difficult for these models to learn about the long-range dependencies in the data. Several techniques were developed that attempt to mitigate this issue, including orthogonal/unitary RNNs~\citep{arjovsky2016unitary, helfrich2018orthogonal}, and gating mechanisms such as long short-term memory (LSTM)~\citep{hochreiter1997long} and gated recurrent units (GRUs)~\citep{cho2014properties}.
Nonetheless, these models are still slow to optimize due to the inherently sequential nature of their computation \citep{kalchbrenner2016neural}, and are therefore hard to scale.

In recent years, Transformers \citep{vaswani2017attention}
have gained increasing prominence for sequence modelling tasks, achieving remarkable success in a wide range of applications \citep{brown2020language, dosovitskiy2020image, jumper2021highly}. Compared to RNNs, attention layers are easier to scale and parallelize during training, and crucially they do not suffer from the vanishing gradient problem, since the interaction between any two tokens in the sequence is modeled by direct edges in the network. A key issue with attention layers however is that their computational and memory costs scale quadratically as $O(L^2)$ with the sequence length $L$. Transformers can therefore be especially expensive to deploy on long sequences. RNNs, which scale linearly with the sequence length, are therefore typically faster than transformers at inference time even for modest sequence lengths \citep{liu2019mkd}.

Motivated by these problems,
\citet{gu2021efficiently} recently introduced the S4 model, a carefully designed deep state-space model~(SSM) achieving remarkable performance on tasks from the Long Range Arena (LRA) \citep{tay2020long}, a benchmark explicitly designed to require very long-ranged reasoning
. S4 is theoretically principled and inspired by continuous-time linear SSMs; well-established components of modern control systems. More importantly, the S4 layer and its variants~(DSS, S4D, S5, etc)~\citep{gupta2022diagonal,gu2022parameterization,smith2022simplified} overcome the $O(L^2)$ bottleneck of attention layers by modeling interactions between tokens using a hidden state~(like RNNs) 
under proper discretization techniques.
These models can be made very efficient at inference time by simply unrolling the layer like an RNN. Futhermore, since SSMs are linear in the temporal dimension, they are easily parallelizable during training, in contrast to the slow sequential nature of training a typical RNN. This makes them very computationally efficient on long sequences.

While the S4 model is equivalent to an RNN during inference, it has a number of unique characteristics during training. For example, S4 is parameterized as a discretization of a latent continuous-time system of differential equations. S4 also uses specific initializations of the state matrices motivated from the theory of polynomial projections \citep{gu2020hippo}. While these characteristics might seem to motivate the impressive performance of these models, later works \citep{gu2022parameterization,smith2022simplified,gupta2022diagonal,gupta2022simplifying} have suggested that the specific initialization used by S4 is often not crucial for performance, and that the discretization rules which achieve best performance  may deviate from theory \citep{smith2022simplified}. It is therefore unclear what these unique characteristics of the deep SSMs are doing mechanistically, and how they can be simplified.

\begin{figure}
\centering
    \includegraphics[width=0.99\textwidth]{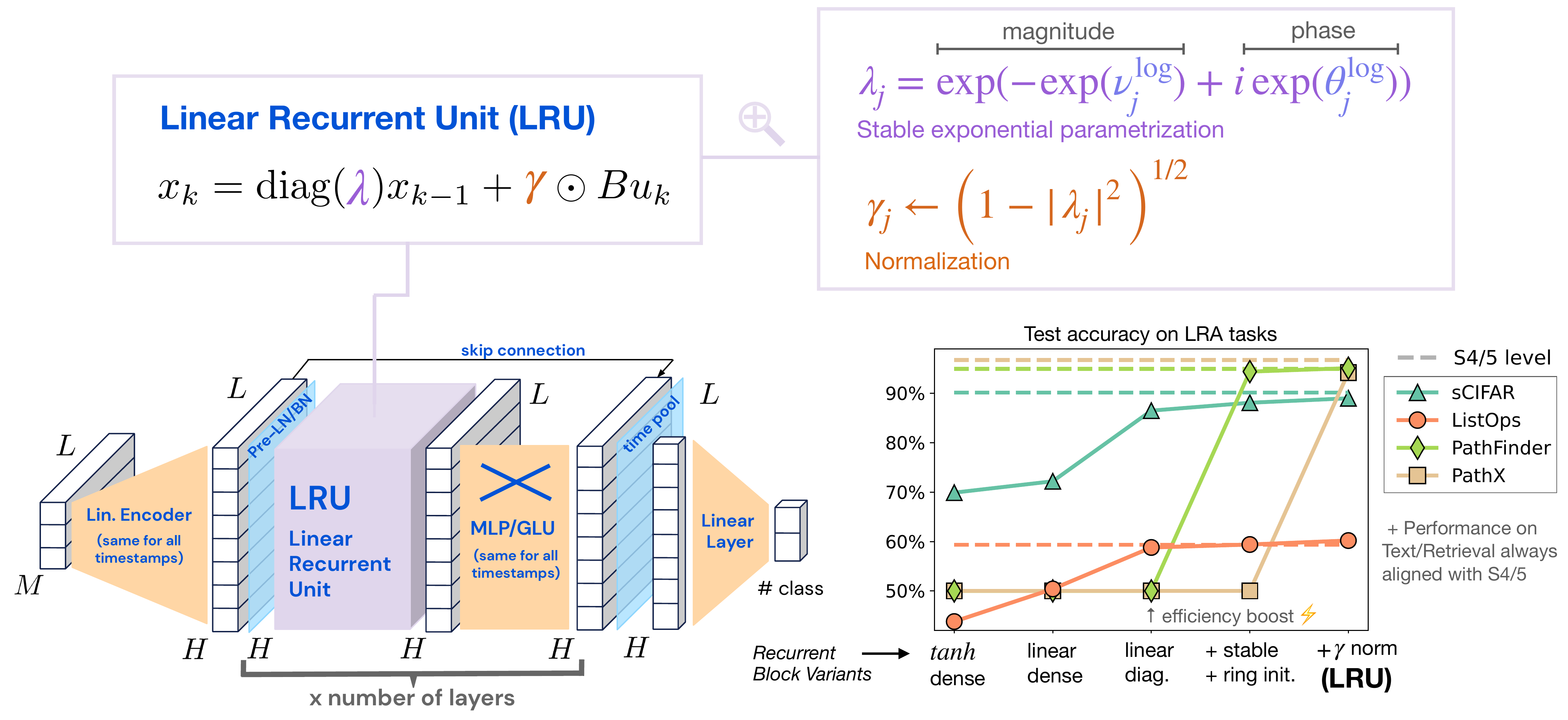}
    \caption{\textit{\textbf{(Left)} Deep Linear Recurrent Unit~(LRU) architecture introduced in this paper, inspired by S4~\citep{gu2021efficiently}. The model is a stack of LRU blocks, with nonlinear projections in between, and also uses skip connections and normalization methods like batch/layer normalization. We expand on the details in \S\ref{app:detailed_experimental_setup} and provide pseudocode in \S\ref{app:pseudocode}. We also use the same architecture structure~(Norm-Recurrence-GLU-Skip) for every variant of the recurrent module in our study~($\tanh$ dense, linear dense, etc..). \textbf{(Right)} Summary of effects for the main steps outlined in the introduction towards designing LRUs starting from $\tanh$ RNNs. Shown is the average performance~(3 seeds) of the recurrent module at each step on the Long Range Arena~(LRA), compared to average performance of deep SSMs.
    For all LRA tasks, we match the performance of deep SSMs like S4/S4D/S5 with LRUs. Detailed results in \S\ref{sec:how_to}.
    }}
    \label{fig:deep_rnn}
\end{figure}

Motivated by the striking similarities between RNNs and deep SSMs, and in an attempt to better understand the underlying mechanism driving the performance of these models, we study the power and limitations of RNNs when used as core components of deep architectures for long-range reasoning. Our main goal is to answer the question:
\vspace{-2mm}
\begin{center}
    ``\textit{Can we match the performance and efficiency of deep continuous-time SSMs using deep RNNs?}''
\end{center}
\vspace{-1mm}
We give a \textit{positive answer} to this question. We show that the performance boost provided by deep SSMs like S4 can also be achieved via a series of small changes to a vanilla deep RNN. With these changes, we can recover the performance and efficiency of these deep SSMs on the Long Range Arena (LRA) benchmark \citep{tay2020long}. We call this new RNN model the Linear Recurrent Unit (or LRU for short).

 \vspace{-2mm}
\paragraph{Main Steps.}  We outline here the main steps needed towards crafting performant and efficient RNN models. 
Note while some of these observations have been made in prior works (see \S\ref{sec:related}),
we provide novel perspectives and careful ablations leading to new insights. Each step presented in this paper unveils a specific property of recurrent networks, and showcases the challenges and best practices in training and initializing deep RNNs. 

\begin{itemize}[leftmargin=4mm, itemsep=2mm]
\vspace{-2mm}
\item \textbf{Linear Recurrences.} When replacing SSM layers in a deep architecture with vanilla RNN layers using tanh or ReLU activations,
the performance on Long Range Arena (LRA) drops significantly. Surprisingly, in \S\ref{sec:linear_rnns} we find that simply \textit{removing the nonlinearities in the recurrence of the RNN}~(i.e., using linear recurrences) gives a substantial boost in test accuracy. We motivate this effect in~\S\ref{app:expressivity} by showing that stacking linear RNN layers and nonlinear MLP blocks (Fig.\ref{fig:deep_rnn}) can indeed model complex nonlinear sequence-to-sequence maps without the need for nonlinearities in the recurrence. While dropping the nonlinearity does not seem to harm expressivity, it leads to several advantages, from the ability to directly control how quickly the gradients might vanish or explode, to allowing us to parallelize training. Our findings also partially motivate the success of deep SSMs, where the recurrence is also linear.

\item \textbf{Complex Diagonal Recurrent Matrices.} 
Dense linear RNN layers can be reparameterized to a complex diagonal form without affecting the expressivity of the network or the features at initialization (\S\ref{sec:diagonalization}). Diagonal linear RNN layers additionally allow for a \textit{highly parallelizable unrolling} of the recurrence using parallel scans to substantially improve training speeds~\citep{martin2017parallelizing}. 
We validate that these observations, which have been leveraged by prior SSMs \citep{gupta2022diagonal,smith2022simplified},
also provide important efficiency improvements for linear RNN layers.
\looseness=-1

\item \textbf{Stable Exponential Parameterization.} In \S\ref{sec:exponential} we show that using an exponential parameterization for the diagonal recurrent matrix has important benefits. Crucially, this enables us to easily enforce stability during training, which in turn allows us to modify the initialization distribution to facilitate long-range reasoning and improve performance. Our results indicate that rather than the specific deterministic initializations used by several recent SSMs, it is the eigenvalue distribution of the recurrent layer at initialization that determines if the model can capture long-range reasoning.

\item \textbf{Normalization.} In \S\ref{sec:norm_pathx} we show that normalizing the hidden activations on the forward pass is important when learning tasks with very long-range dependencies. With this final modification, our RNNs can match the performance of deep SSMs on all tasks in the LRA benchmark. 
Connecting back to state-space models, we show in \S\ref{sec:explaining_s4} how our normalization can be linked to the discretization structure in S4.
\end{itemize}

\vspace{-2mm}
We summarize the deep Linear Recurrent Unit (LRU) architecture used in this paper, and the effect of each of the above steps on performance in Fig.\ref{fig:deep_rnn}. 
We emphasize that the main purpose of our work is not to surpass the performance of S4-based models, but rather to demonstrate that simple RNNs can also achieve strong performance on long range reasoning tasks when properly initialized and parameterized.
We believe the insights derived in this paper can be useful to design future architectures, and to simplify existing ones.

\section{Preliminaries}
\label{sec:preliminaries}
\vspace{-1mm}
In this section, we compare the key architectural components (RNNs and SSMs) studied in this work, and also describe our methodology and experimental setup. For a more thorough discussion or related architectures, the reader can check our related work section~\S\ref{sec:related}.
\vspace{-2mm}
\subsection{Recap of recurrent block structures}
\label{sec:recap}
\vspace{-1mm}
We give an overview of the main architectural components considered in this paper, focusing on the major difference between Vanilla RNNs and recent S4-like deep SSMs~\citep{gu2021efficiently,gu2022parameterization, gupta2022diagonal, smith2022simplified}.
\vspace{-3mm}
\paragraph{RNN Layer.}Let $(u_1, u_2,\dots, u_L)$ be a sequence of $H_{\text{in}}$-dimensional inputs, which can be thought of as either the result of intermediate layer computations~(which keep the sequential structure) or as the initial input. An RNN layer with $N$-dimensional hidden state computes a sequence of $H_{\text{out}}$-dimensional outputs $(y_1, y_2,\dots, y_L)$ through a recurrent computation\footnote{We do not use bias parameters as they can be incorporated into the MLP blocks preceding and following the RNN block. Classical RNNs also included a nonlinearity on the output $y_k = \sigma_{\text{out}}(C x_k + b)$ with $D=0$. Having $D\ne 0$ basically introduces a skip connection~(standard in modern architectures), and the $\sigma_{\text{out}}$ can be thought of as part of the MLP following the RNN.}
using learnable parameters $A\in\R^{N\times N}, B\in\R^{N\times H_{\text{in}}}, C\in\R^{H_{\text{out}}\times N}, D\in\R^{H_{\text{out}}\times H_{\text{in}}}$:
\begin{align}
    &x_{k} = \sigma(Ax_{k-1}+ Bu_k),\quad y_k = C x_k + D u_k, 
    \label{eq:RNN}
\end{align}
starting from $x_0 = 0\in\R^N$. $\sigma$ here denotes a nonlinearity, often chosen to be a $\tanh$ or sigmoid activation. If $\sigma$ is the identity function, then we say the RNN layer is \textit{linear}. 
\vspace{-3mm}
\paragraph{S4-like recurrent layer.} We present a simplified\footnote{This version is most similar to S5~\citep{smith2022simplified}, but is here presented for ease of reasoning for a single discretization parameter $\Delta$, shared across input dimensions. For more details, see~\S\ref{sec:related}.
} version of the S4 recurrence introduced in~\citet{gu2021efficiently}. The input $(u_0, u_1,\dots, u_{L-1})$ is now seen as the result of sampling a latent continuous-time signal $ u_{\text{ct}}:\R_{\ge0}\to\R^{H_{\text{in}}}$ at multiples of a stepsize $\Delta>0$: i.e. $ u_{\text{ct}}(\Delta k) := u_k$ for all $k\in 0,\dots, L-1$. The output sequence $(y_0, y_1,\dots, y_{L-1})$ is then sampled, again with stepsize $\Delta$, from the signal $y_{\text{ct}}:\R_{\ge 0}\to \R^{H_{\text{out}}}$ computed by the following continuous-time state-space model, initialized at $x_{\text{ct}}(0) = 0$:
\vspace{-1mm}
\begin{align}
    &\frac{d}{dt}x_{\text{ct}}(t) = \tilde A x_{\text{ct}}(t) + \tilde B u_{\text{ct}}(t),\nonumber\\
    &y_{\text{ct}}(t) =  \Re\left[\tilde C  x_{\text{ct}}(t)\right] + \tilde D u_{\text{ct}}(t),
    \label{eq:S4}
\end{align}
where $\Re(p)$ denotes the real part of a complex-valued vector $p$, $\tilde A = \text{diag}(\tilde a)$ with $\tilde a\in\Cmp^N$, $\tilde B\in\Cmp^{N\times H_{\text{in}}}, \tilde C\in\Cmp^{H_{\text{out}}\times N}$ and $\tilde D\in\R^{H_{\text{out}}\times H_{\text{in}}}$. Ignoring the continuous-time nature of this model, the most striking differences compared to Eq.\eqref{eq:RNN} are that (a) the computation on the right-hand-side is \textit{linear} in the hidden state and in the input, and (b) most parameters are \textit{complex} valued, with $\tilde A$ being diagonal. While $\tilde B,\tilde C,\tilde D$ follow complex random or uniform initialization, the transition matrix $\tilde A$ is \textit{structured}, i.e., initialized \textit{deterministically} through HiPPO theory~\citep{gu2020hippo} in diagonal form. Common choices~\citep{gu2022parameterization} are $\tilde a_n = -\frac{1}{2}+ i\pi n$ (S4D-Lin) and 
$\tilde a_n = -\frac{1}{2}+ i\frac{N}{\pi}\left(\frac{N}{n+1}-1\right)$~(S4D-Inv), for $n = 1,2,\dots, N$. 

For training and inference, the continuous-time system in Eq.\eqref{eq:S4} is discretized at stepsize $\Delta$ through a high-accuracy Zero-Order-Hold~(ZOH) or Bilinear method. The ZOH method gives 
\begin{align}
    &x_{k} = Ax_{k-1}+ Bu_k, \quad y_k = C x_k + D u_k, 
    \label{eq:S4-disc}
\end{align}
where $x_{-1}=0$, $A = \exp(\Delta \tilde A)$, $B = (A-I)\tilde A^{-1}\tilde B$, $ C = \tilde C$ and $ D = \tilde D$, and $\exp$ denotes the matrix exponential. Under the assumption that $u_{\text{ct}}$ is constant in between timestamps~(which can be thought of as a modeling assumption), this numerical integration is \textit{exact}~\citep{jacquot2019modern}. Moreover, note that all these discretization operations can be quickly performed element-wise since $\tilde A$ is diagonal.

\vspace{-3mm}
\paragraph{Some key differences.} It is worth pointing out a few structural and computational properties, to highlight some crucial differences between RNNs and SSMs:
\vspace{-2.5mm}
\begin{itemize}[leftmargin=4mm,itemsep=0.5mm]
    \item Since Eq.\eqref{eq:S4-disc} is linear, it can be efficiently parallelized until $k=L-1$ using parallel scans~\citep{martin2017parallelizing, smith2022simplified}, unlike a nonlinear RNN 
    where the computation has to be performed sequentially.
    \item While Eq.\eqref{eq:S4-disc} is similar to the linear RNN computation, it is crucial to note that (a) $A$ and $B$ are parameterized in a peculiar way, prescribed by discretization, and (b) these matrices share parameters; in particular $\Delta$ affects both $A$ and $B$. These differences are critical as in SSMs learning is performed on the continuous-time parameters $\tilde A, \tilde B, \tilde C, \tilde D, \Delta$; hence parameterization choices directly affect optimization.
    \item Unlike vanilla RNNs, most SSMs use complex-valued diagonal recurrent matrices that are initialized deterministically using HiPPO theory, and the literature attributes much of the success of SSMs to the specific initialized used \citep{gu2021efficiently, gupta2022diagonal, gu2022train}.
\end{itemize}
\vspace{-2mm}
The points above motivate our investigation: in this paper we consider the same architecture as~\citet{gu2021efficiently,gu2022parameterization,smith2022simplified}, but replace the SSM layer in the recurrent core by an RNN. We then study which steps need to be taken to gradually retrieve S4-like performance on LRA~\citep{tay2020long} tasks. The effectiveness of each of our steps is supported by empirical evidence and theoretical considerations, and leads to the architecture presented in Fig.\ref{fig:deep_rnn}.

\subsection{Experimental setup}
\label{sec:exp_setup}
In this paper, we consider the Long Range Arena benchmark \citep{tay2020long}, a set of tasks designed to test the ability of models to do long-range sequence modelling (except we use coloured images instead of grayscale images for the sequential CIFAR-10 classification task). Transformers fail to perform well on most of these tasks, while deep SSMs have shown remarkable performance on these tasks \citep{gu2021efficiently, dao2022flashattention}. This makes it an appropriate benchmark to explore the long-range modelling capabilities of deep RNNs.

For all our experiments, we use a network of 6 layers with residual connections and layer/batch normalization~\citep{ioffe2015batch,ba2016layer} similar to \citet{gu2021efficiently} (Fig.\ref{fig:deep_rnn}), and we replace the SSM layers with RNN layers, building up to our LRU recurrence in a sequence of steps~(see \S\ref{sec:how_to}). 
All experiments are repeated three times, and we report the mean and standard error. Networks are trained using the AdamW optimizer~\citep{loshchilov2017decoupled}. We use a smaller learning rate and no weight decay on the recurrent parameters, as suggested by~\citet{steil2004backpropagation, gu2021efficiently}. We tune hyperparameters such as learning rates for all models on a logarithmic grid for best accuracy.
See \S\ref{app:detailed_experimental_setup} for more details on our experimental setup.

\section{Designing Performant Deep RNNs}
\label{sec:how_to}
In this section, we discuss the fundamental steps needed for designing RNNs to reach the impressive performance of deep SSMs on the LRA benchmark. We present these steps, already outlined in the introduction, in logical order, and support each claim with experimental evidence and theoretical considerations, expanded in \S\ref{app:theory}.

We consider the architecture of Fig.\ref{fig:deep_rnn}, where the recurrent computation is gradually modified starting from a vanilla RNN. We start by showcasing the advantage of using linear recurrences in \S\ref{sec:linear_rnns}; then, in \S\ref{sec:diagonalization}, we show how to speed-up training and inference without affecting expressivity and initialization distribution. In \S\ref{sec:exponential}, we discuss how~(and why) changing the parameterization and initialization distribution enables us to make the RNN stable and improve long-range modeling. Finally, in \S\ref{sec:norm_pathx}, we finalize the LRU architecture by proposing a normalization strategy for the hidden activations that results in a close match in performance with deep SSMs.

\subsection{Linear RNN layers are performant}
\label{sec:linear_rnns}
One of the main findings of our work is that linear RNN layers can be surprisingly expressive when coupled with nonlinear MLP or GLU~\citep{dauphin2017language} blocks, outperforming tuned nonlinear RNN variants in the same architecture. In Tb.\ref{tb:effect_nonlinearity}, we show that simply removing\footnote{All other settings in the recurrent block are kept the same as in the Vanilla RNN module of Haiku~\citep{haiku2020github}. That is, all matrices have Glorot~\citep{glorot2010understanding} initialization. The rest of the architecture is kept as in Fig.\ref{fig:deep_rnn}, where the LRU block is replaced by an RNN.} the nonlinearity, and therefore computing the next state as $x_{k} = A x_{k-1} + B u_k$, is able to improve test accuracy on most LRA tasks. While the boost provided by vanilla linear RNN blocks leads to performance which is still far behind S4 on some tasks~(sCIFAR, PathFinder and PathX), this first finding motivates us to drop nonlinearities in the recurrence for the rest of this paper. In later sections, we leverage the linearity of the recurrence to significantly speed up training as well as derive principled initialization and normalization principles to learn long-range dependencies. We note that, on the Text and Retrieval tasks, performance using vanilla RNNs already matches performance of deep SSMs~(see Tb.\ref{tb:effect_normalization} for the performance of S4D/S5 on these tasks).

\vspace{2mm}
\begin{table}[ht]
\begin{center}
\begin{footnotesize}
\begin{sc}
\setlength{\tabcolsep}{10pt}
\begin{tabular}{|l||*{4}{c|}}
\hline
Recurrence & \textbf{sCIFAR} & \textbf{ListOps} & \textbf{Text} &\textbf{Retrieval} \TBstrut\\
\hline
RNN-ReLU & 69.7 (0.2)  & 37.6 (8.0)  & 88.0 (0.1)  &  88.5 (0.1) \TBstrut\\
RNN-Tanh & 69.9 (0.3)  & 43.9 (0.1)  &  87.2 (0.1) & 88.9 (0.2)  \TBstrut\\
RNN-Lin   & \textbf{72.2} (0.2) & \textbf{50.4} (0.2) & \textbf{89.1} (0.1)  & \textbf{89.1} (0.1)  \TBstrut\\
\hline
\end{tabular}
\end{sc}
\end{footnotesize}
\end{center}
\vskip -0.1in
\caption{\textit{The effect of removing the nonlinearity from the recurrent unit on test accuracy (\S\ref{sec:linear_rnns}). 
We show here results only for the sCIFAR, ListOps, Text and Retrieval tasks in LRA as these models did not exceed random guessing on PathFinder/PathX~(further improvements in Tb.\ref{tb:effect_parametrization_exp} and \ref{tb:effect_normalization}). Performance of deep SSMs shown in Tb.\ref{tb:effect_normalization}.}
}
\label{tb:effect_nonlinearity}
\end{table}

\vspace{-1mm}
The empirical result in Tb.\ref{tb:effect_nonlinearity} is \textit{surprising}, since recurrent nonlinearities are believed to be a key component for the success of RNNs --- both in the theory and in practice~\citep{siegelmann2012neural, pascanu2013difficulty,erichson2021lipschitz}. Indeed, a strong property of single-layer sigmoidal and $\tanh$ RNNs is Turing completeness, which cannot be achieved by the linear variant~\citep{chung2021turing}. However, the architecture we use (Fig.\ref{fig:deep_rnn}) is deeper than a standard RNN and includes nonlinearies, placed position-wise \textit{after} each RNN block. In~\S\ref{app:expressivity}, we investigate how the expressivity and trainability of deep models is affected by recurrent nonlinearities. Leveraging a spectral analysis and Koopman operator theory~\citep{koopman1932dynamical}, we discuss how interleaving linear RNN layers with nonlinear feedforward blocks is sufficient to approximate highly nonlinear systems. A key observation in our analysis is that position-wise nonlinearities effectively transfer signal information to higher frequencies, enabling the system to go beyond linearity in the spectral domain and increasing the layer capacity. To further strengthen our claim on the advantage of linear recurrences, in \S\ref{app:opt} we show that, while linear and nonlinear RNNs share an important class of approximating functionals~(linear operators, see~\cite{wangeffects}),  nonlinear activations can potentially slow down training.

\subsection{Using complex diagonal recurrent matrices is efficient}
\label{sec:diagonalization}
We now show that we can significantly speed up training and inference for deep linear RNNs without losing performance by using complex-valued diagonal recurrent matrices. While the idea of diagonalizing linear systems for computational efficiency is a dominating feature of all deep SSMs since the introduction of DSS by~\citet{gupta2022diagonal}, in this section we construct our diagonalized version to exactly match the initialization spectrum~(see \S\ref{sec:eig}) of the Glorot-initialized deep linear RNN in Tb.\ref{tb:effect_nonlinearity}. Our main purpose with this approach is to \textit{disentangle the effects of initialization and diagonalization on performance}~(cf. Tb.\ref{tb:effect_parametrization_exp} and Tb.\ref{tb:effect_normalization}).

We start in \S\ref{sec:eig} by recalling some useful linear algebra elements, and then proceed in \S\ref{sec:diagonalization_learning} with a discussion on how to diagonalize the recurrence while preserving the eigenvalue spectrum at initialization.

\subsubsection{Linear RNN eigendecomposition}
\label{sec:eig}
The recurrence $x_k = Ax_{k-1} + B u_k$ can be unrolled easily using the assumption that $x_{-1}=0\in\R^N$:
\begin{equation}
    x_{0} = Bu_0,\quad x_1 = ABu_0 + Bu_1,\quad x_2 = A^2Bu_0 + ABu_1 + B u_2, \quad \dots \quad \implies \quad x_k = \sum_{j=0}^{k-1} A^jBu_{k-j}.
    \label{eq:lin_rnn_unroll}
\end{equation}
Exponentiations of the matrix $A$ in the equation above are the source of the well-known \textit{vanishing/exploding gradient} issue in RNNs~\citep{bengio1994learning,pascanu2013difficulty}. While in nonlinear RNNs the state $x_k$ is forced to live on the compact image of the activation function, the hidden-state of our linear variant can potentially explode or vanish
exponentially as $k$ increases. This phenomenon can be better understood by leveraging an eigenvalue~(a.k.a.~spectral) analysis: up to an arbitrarily small perturbation of the entries, every matrix $A\in\R^{N\times N}$ is diagonalizable\footnote{In other words, the set of non-diagonalizable matrices has measure zero, see e.g.~\citet{zhinan2002jordan} for a proof idea.}~\citep{axler1997linear}, i.e.~one can write $A = P\Lambda P^{-1}$, where $P\in\Cmp^{N\times N}$ is an invertible matrix and $\Lambda = \text{diag}(\lambda_1,\lambda_2,\dots,\lambda_N)\in\Cmp^{N\times N}$. It is essential to note that, unlike the symmetric setting where eigenvalues and eigenvectors are real, in the non-symmetric case\footnote{Take e.g. $A = ((0,1)(-1,0))$. The solution to the standard eigenvalue equation gives $\lambda = \pm i$, where $i$ is the imaginary unit.} one has to allow for \textit{complex} entries to achieve full equivalence. Plugging the decomposition $A=P\Lambda P^{-1}$ into Eq.\eqref{eq:lin_rnn_unroll} and multiplying both sides by $P^{-1}$, we get $\bar x_k = \sum_{j=0}^{k-1} \Lambda^j\bar B u_{k-j}$, where $\bar x_k := P^{-1}x_k$, $\bar B := P^{-1}B$. The output can then be computed as $y_k = \Re[\bar C \bar x_k] + D u_k\in\R^H$, where $\bar C = CP^{-1}$, and we take the real part of $\bar C \bar x_k$. Therefore, instead of learning $(A,B,C,D)$, one can equivalently learn $(\Lambda, \bar B, \bar C, D)$, where $\Lambda, \bar B, \bar C$ are complex valued, and $\Lambda$ is a diagonal matrix.

\vspace{-3mm}
\paragraph{Are complex numbers really necessary?} We adopt complex numbers since they provide a convenient and compact representation of non-symmetric matrices in diagonal form. However this is not the only option -- one could work~(almost) as efficiently using real numbers. We discuss how this can be achieved in \S\ref{app:real_jordan_form}. 

\vspace{-3mm}
\paragraph{Stability.} 
Since $\bar x_k = \sum_{j=0}^{k-1} \Lambda^j\bar B u_{k-j}$, the norm of component $j$ of $\bar x$ at timestamp $k$  evolves such that $|x_{k,j}|=O(|\bar x_{k,j}|)= O(|\lambda_j|^k)$. Therefore, a sufficient condition to ensure stability~(i.e. $x_k$ does not explode) is therefore $|\lambda_j|<1$ for all $j$ \citep{gu2021efficiently}.

\subsubsection{Learning in the diagonalized space} 
\label{sec:diagonalization_learning}
Learning recurrent linear systems in diagonal form provides substantial computational speedups both for training and inference. For example, in our implementation of sCIFAR, we found diagonal linear RNNs to be $\sim$8 times faster to train than a dense RNN with ReLUs, matching the speed of our implementations of S4D and S5. The main reasons for this computational benefit are that (a) taking powers of diagonal matrices is trivial (speeding up both training and inference), while exponentiating dense matrices is computationally expensive, and (b) while nonlinear recurrences must be computed sequentially, unrolling a linear recurrence can be parallelized using associative scans resulting in faster training \citep{smith2022simplified, gupta2022diagonal}.

\vspace{-3mm}
\paragraph{Equivalent initialization.} To disentangle the benefits of diagonal linear systems from the role of initialization, we seek an initialization for the diagonal system which keeps the eigenvalue spectrum of the recurrence unchanged when comparing our diagonal system with the dense linear RNN in \S\ref{sec:linear_rnns}, where $A$ followed Glorot initialization. Fortunately, we can use a classical result from random matrix theory~\citep{ginibre1965statistical}.

\begin{restatable}[Strong circular law]{thm}{ginibre}
Let $\mu_N$ be the empirical spectral measure of $A_N$, where $A_{N}$ is a real $N\times N$ matrix with i.i.d. Gaussian entries, each with zero mean and variance $1/N$. Then, $\mu_N$ converges weakly almost
surely as $N\rightarrow \infty$ to the uniform probability measure on $\{|z| \le 1\} \subseteq \mathbb{C}$.
\label{thm:ginibre}
\end{restatable}

\begin{figure}
    \centering
    \begin{minipage}{0.66\textwidth}
    \vspace{-3mm}
    \centering
        \includegraphics[height=36mm]{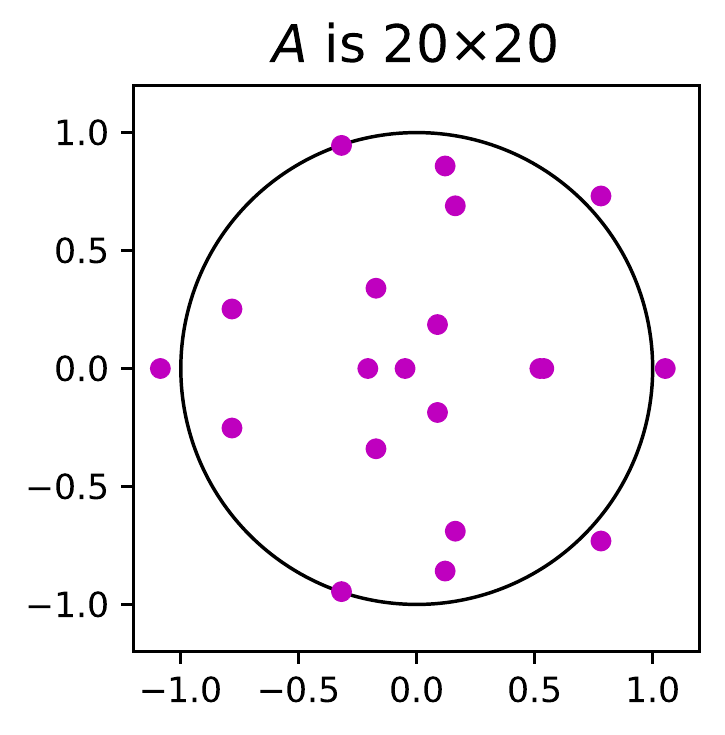}
        \includegraphics[height=36mm]{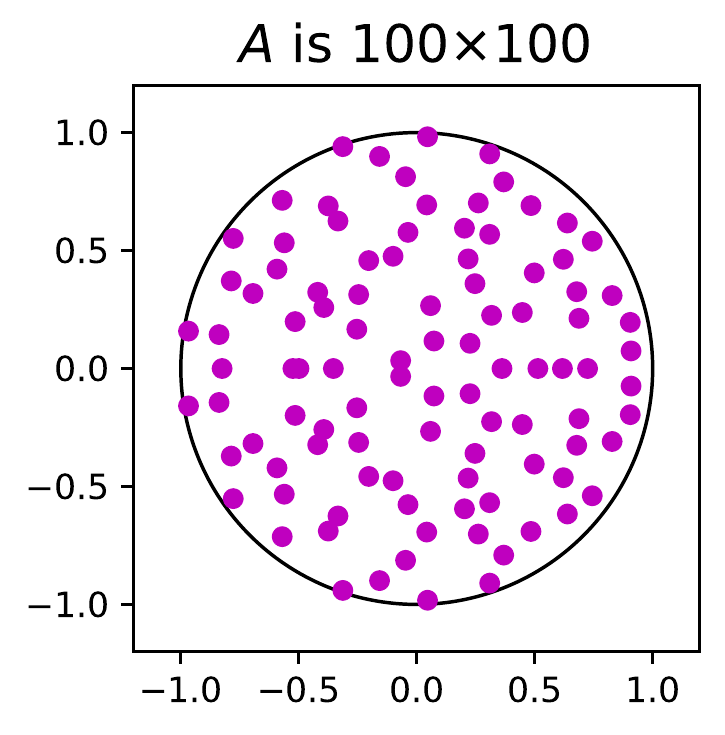}
        \includegraphics[height=36mm]{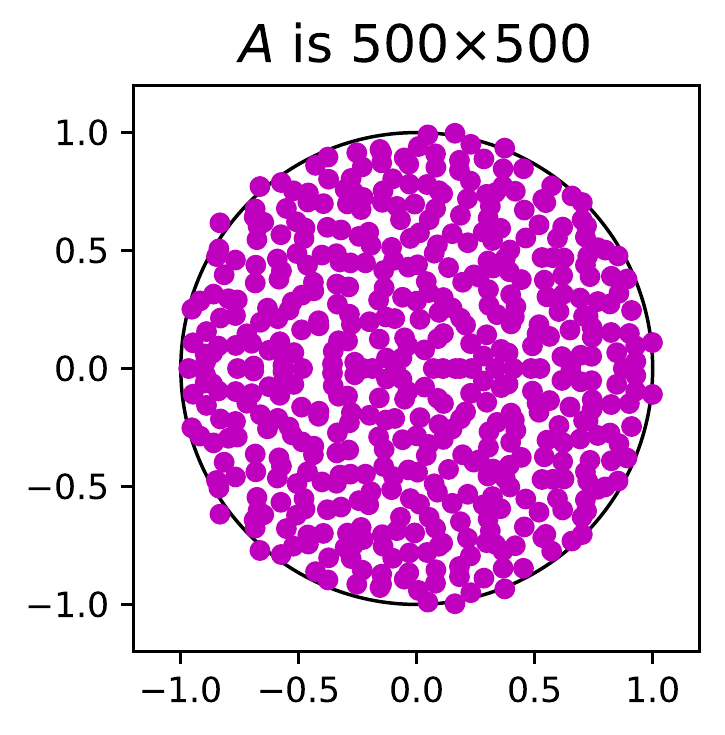}
        \vspace{-2mm}        \caption{ \textit{Eigenvalues of $A\in\R^{N\times N}$ following Glorot initialization: each entry of $A$ is sampled independently from a Gaussian with mean 0 and variance $1/N$. The eigenvalues are complex~($A$ is not symmetric) and are represented on the complex plane. The black circle is the unit disk $\{|z| = 1\} \subseteq \mathbb{C}$. The limit behavior~(uniform initialization) is predicted by Thm.~\ref{thm:ginibre}.}}
        \label{fig:ginibre}
    \end{minipage}
    \hspace{4mm}
    \begin{minipage}{0.29\textwidth}
    \vspace{-3mm}
    \centering
    \includegraphics[height=36mm]{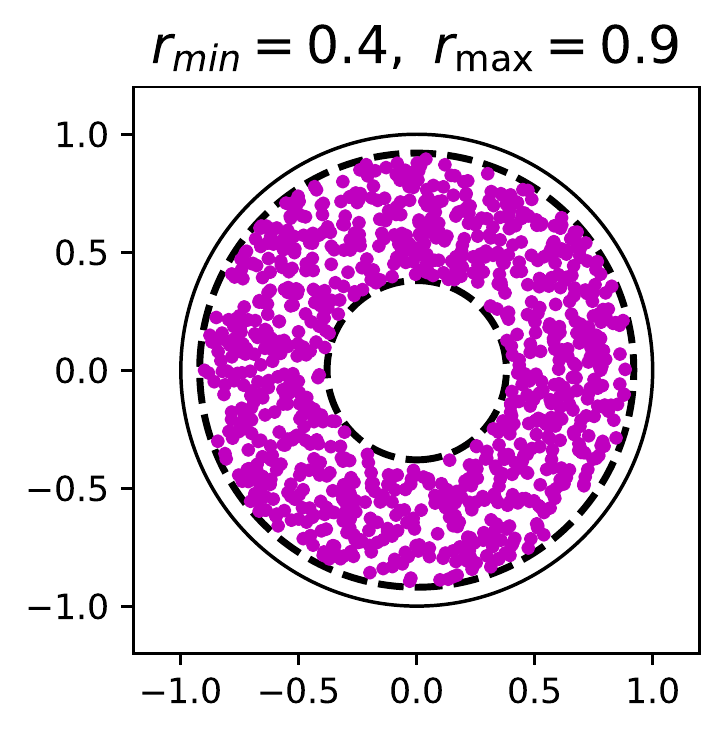}
        \vspace{-2mm}
    \caption{ \textit{Eigenvalues of a diagonal matrix $A$ with entries sampled using Lemma~\ref{lemma:sampling_exp}. For $r_{\min}=0$, $r_{\max}=1$, the distribution coincides with Glorot init. in the limit. }}
    \end{minipage}
\end{figure}
\vspace{-2mm}
The theorem above, illustrated in Fig.\ref{fig:ginibre}, shows that under Glorot initialization the spectrum of $A$ is \textit{de-facto} sampled from the unit disk in $\Cmp$. This result motivates the strong performance of linear RNNs in \S\ref{sec:linear_rnns}, since it implies Glorot initialization provides an approximately stable initialization~(see definition in \S\ref{sec:eig}).\footnote{Later in training, the system is less likely to become unstable if the learning rate is small enough. 
} Moreover, from Theorem~\ref{thm:ginibre}, an \textit{equivalent spectral initialization} follows for the diagonal system, which holds exactly for the large width limit: $\Lambda$ should be diagonal with entries sampled uniformly on the unit disk. Using the definition of exponential of a complex number: $\exp(-\nu + i\theta) := e^{-\nu}(\cos(\theta) + i \sin(\theta))$, we adopt a simple scheme for sampling uniformly on a ring in between circles with radii $r_{\min}$ and $r_{\max} $in $\Cmp$.

\begin{restatable}[]{lem}{sampling}
 Let $u_1,u_2$ be independent uniform random variables on the interval $[0,1]$. Let $0\le r_{\min}\le r_{\max}\le1$. Compute $\nu = -\frac{1}{2}\log\left(u_1(r_{\max}^2-r_{\min}^2)+r_{\min}^2\right)$ and $\theta = 2\pi u_2 $. Then $\exp(-\nu+i\theta)$ is uniformly distributed on the ring in $\mathbb{C}$ between circles of radii $r_{\min}$ and $r_{\max}$.
 \label{lemma:sampling_exp}
 \end{restatable}
We recover the spectrum of Glorot-initialization (in the limit of infinite width) by setting $r_{min} = 0$ and $r_{max} = 1$ (we will explore tuning these hyper-parameters in \S\ref{sec:exponential}). Tb.\ref{tb:effect_parametrization_exp}~(first two rows) shows the results of learning deep linear RNNs in complex diagonal form,\footnote{To avoid issues with backpropagation on complex variables, each complex parameter in the network is stored and learned as a pair of floats encoding real and imaginary parts.} where each diagonal entry of $\Lambda$ is initialized uniformly on unit disk in $\Cmp$ using Lemma~\ref{lemma:sampling_exp} with $[r_{\min}, r_{\max}]=[0,1]$. In our experiments, $\bar B, \bar C$~(which we rename for convenience back to $B$ and $C$) follow Glorot initialization for both real and imaginary parts~(parameterized separately), with halved variance in each component to preserve lengths on the input-output projections~\citep{glorot2010understanding}. Finally, after the SSM computation, the real part of the signal is kept and the imaginary discarded~(as in~\cite{gupta2022diagonal, gu2022parameterization}). 

Our results in Tb.\ref{tb:effect_parametrization_exp} show that diagonalizing the recurrence surprisingly improves accuracy on tasks like ListOps and sCIFAR. More importantly, it drastically reduces training and inference time on all LRA tasks (see Tb.\ref{tb:steps_per_sec} in \S\ref{app:training_speedups} for training speed comparisons), and makes the RNN just as fast to train as deep SSMs like S4D and S5.

\subsection{Benefits of stable exponential parameterization}
\label{sec:exponential}
In \S\ref{sec:diagonalization} we showed that moving to complex diagonal recurrences is computationally efficient. However we also observed that learning the diagonal model can be more unstable than learning the dense model in some experiments. To learn long-range dependencies and avoid quickly vanishing gradients, eigenvalues in the recurrence need to have magnitude close to 1 \citep{gu2022train, gupta2022diagonal}; however, these eigenvalues are also likely to make the system unstable during training. In this section, we show the benefits of a stable parameterization of the RNN, and of tuning $r_{\min}$ and $r_{\max}$~(see Lemma~\ref{lemma:sampling_exp}).
\vspace{-3mm}

\paragraph{Optimization under exponential parameterization.} Lemma~\ref{lemma:sampling_exp} suggests a natural parameterization of the diagonalized RNN as $\Lambda = \diag(\exp(-\nu + i \theta))$ with $\nu\in\R^N$ and $\theta\in\R^N$ as the learnable parameters (instead of the real and imaginary parts of $\Lambda$). As we explain in~\S\ref{app:opt} leveraging an easy-to-visualize 2-dimensional example~(see Fig.\ref{fig:learning_powers}), this choice decouples magnitude and oscillation frequencies, making optimization with Adam easier. The positive effects of this exponential parametrization, which resembles some features of ZOH discretization~(see \S\ref{sec:preliminaries} and \S\ref{sec:explaining_s4}) and notably takes the performance of PathFinder above random chance, can be observed in the third row of Tb.\ref{tb:effect_parametrization_exp}.

\vspace{-3mm}
\paragraph{Enforcing stability.} An important benefit of the exponential parameterization is that it makes it simple to enforce stability on the eigenvalues. To see this, note that at initialization, $|\lambda_j| = |\exp(-\nu_j)|\le 1$ since $\nu_j>0$. Therefore, to preserve stability during training, we can use an exponential or another positive nonlinearity: $\lambda_j:=\exp(-\exp(\nu_j^{\log})+i\theta_j)$, where $\nu^{\log}\in\R^N$ is the parameter we optimize, and we set $\nu_j^{\log} := \log(\nu)$ at initialization. Note that a similar idea is used in deep SSMs~\citep{gu2021efficiently} in the context of discretization. We choose an exponential non-linearity over a simple ReLU nonlinearity to increase granularity around $|\lambda| = 1$, achieved at $\nu^{\log} = -\infty$~(while $|\lambda| =0$ is achieved at $\nu^{\log} =\infty$).
Stable parameterization helps on most LRA tasks. In the fourth row of Tb.\ref{tb:effect_parametrization_exp}, we show its effects on sCIFAR, ListOps and Pathfinder. 
We observe the most drastic improvement on Pathfinder, one of the harder long-range dependency tasks in LRA, where performance now reaches above $93\%$.

The benefits of the stable parameterization becomes more apparent when we explore the idea of initializing the eigenvalues of $\Lambda$ on a ring closer to the unit disk~(increasing $r_{\min}$ closer to $1$ in Lemma~\ref{lemma:sampling_exp}) to bias the network towards longer-range interactions and avoid vanishing gradients. Indeed, as discussed in detail in~\citet{gu2022train, gupta2022diagonal}, for reasonings requiring consideration of interactions between distant tokens, eigenvalues in the recurrence need to have magnitude close to $1$. Otherwise, as clear from the diagonal version of Eq.\eqref{eq:lin_rnn_unroll}, when taking powers of eigenvalues close to the origin, the signal from past tokens quickly dies out~(see \S\ref{sec:eig}). As we show in the last row of Tb.\ref{tb:effect_stability_gamma_cifar} in \S\ref{app:additional_results}, without enforcing stability, performance starts to degrade as we increase $r_{\max}$ past 0.9 in the sCIFAR task. With stability enforced, we can increase $r_{\max}$ up to 0.99 and improve performance. We see similar benefits on the other tasks where we sweep different values of $r_{\min}$ and $r_{\max}$ 
(Tbs.\ref{tb:effect_parametrization_exp_stddevs}\ \& \ref{tb:effect_normalization_stddevs} have more details).
Finally, note that while here we explore changing the magnitude of the eigenvalues of $\Lambda$, in \S\ref{sec:norm_pathx} we also show the benefits of initializing the eigenvalues to have a small phase to learn more global patterns, useful for particularly long-range reasoning tasks.

\begin{table}
\begin{center}
\begin{footnotesize}
\begin{sc}
\setlength{\tabcolsep}{10pt}
\begin{tabular}{|l||*{6}{c|}}
\hline
 & \textbf{sCIFAR} & \textbf{ListOps} & \textbf{Pathfinder} \TBstrut\\
\hline
Dense $A$ & 72.2 (0.2) & 50.4 (0.2) & \XSolidBrush \TBstrut\\
$\Lambda$ Real + Im  & 86.5 (0.1) & 58.8 (0.3) & \XSolidBrush \TBstrut\\
$\Lambda$ Exp & 85.4 (0.7) & \textbf{60.5} (0.3) & 65.4 (9.0)  \TBstrut\\
$\Lambda$ Stable Exp & 87.2 (0.4) & 59.4 (0.3) & 93.5 (0.5) \TBstrut\\
+ Ring Init & \textbf{88.1} (0.0) & 59.4 (0.3) & \textbf{94.4} (0.3) \TBstrut\\

\hline
\end{tabular}
\end{sc}
\end{footnotesize}
\end{center}
\vskip -0.1in
\caption{\textit{Test accuracy of a linear diagonal complex RNNs under different parametrizations of the transition matrix~(see \S\ref{sec:diagonalization}). Performance directly improves the results in Tb.\ref{tb:effect_nonlinearity}, and showcases the advantage of exponential~(polar) representation of $\Lambda$.  In bold font is the best parametrization option for linear RNN blocks. Ring Init denotes a changed initialization where $r_{\min}$ and $r_{\max}$ are tuned. Performance on the Text and Retrieval tasks is not shown as linear RNNs already align with S4 results~(c.f. Tb.\ref{tb:effect_nonlinearity} with Tb.\ref{tb:effect_normalization}). These models cannot solve PathX yet, and requires normalizing the hidden activations and initializing the eigenvalues of $\Lambda$ with small phase~(see Tb.\ref{tb:effect_normalization}). }}
\vspace{-3mm}
\label{tb:effect_parametrization_exp}
\end{table}


\subsection{Additional considerations for long-range reasoning tasks}
\label{sec:norm_pathx}

Up to this point, our model did not succeed in learning PathX --- the hardest dataset in our benchmark, with a sequence length of $16k$ tokens. In this section, we discuss the additional modifications we need to make to improve our model's ability to learn very long-range dependencies and finalize our LRU model.

\vspace{-3mm}
\paragraph{Normalization.} In \S\ref{sec:exponential}, we initialized the eigenvalues of $\Lambda$ close to the unit disk for better performance on long-range tasks. However, we observed that as we moved $r_{\min}$ and $r_{\max}$ closer to 1, the training loss also started to blow up at initialization~(see Fig.\ref{fig:normalization_smallphase_pathx}). In this section, we first present a result explaining this phenomenon, before deriving a practical normalization scheme for the hidden activations to tackle this problem and further improve performance.

\begin{restatable}[Forward-pass blow-up]{prop}{blowup}
 Let $\Lambda$ be diagonal with eigenvalues sampled uniformly on the ring in $\mathbb{C}$ between circles of radii $r_{\min}<r_{\max}<1$. Then, under constant or white-noise input and Glorot input projection, we have that the squared norm of the state $x_k$ converges as $k\to\infty$ to the following quantity.
 \vspace{-1mm}
 \begin{equation*}
     \E[\|x_\infty\|^2_2] = \frac{1}{r_{\max}^2-r_{\min}^2}\log\left(\frac{1-r_{\min}^2}{1-r_{\max}^2}\right) \E[\| B u\|_2^2].
 \end{equation*}
 \vspace{-3mm}
 \label{prop:r_min_max_effect}
\end{restatable}
This result has the following intuitive form if $r_{\min}=r_{\max}=r$: if we initialize $\rho$-close to the unit disk, the forward pass blows up by a factor $1/\rho$ (since the contributions from previous states take longer to decay): let $\epsilon = r_{\max}^2-r_{\min}^2$ and $\rho = 1-r_{\max}^2$, then:
 \begin{equation}
     \lim_{\epsilon\to 0}\frac{\E[\|x_\infty\|^2_2]}{\E[\|B u\|_2^2]} = \lim_{\epsilon\to 0}\left[ \frac{1}{\epsilon}\log\left(1+\frac{\epsilon}{\rho}\right)\right] = \lim_{\epsilon\to 0} \left[\frac{1}{\epsilon}\left(\frac{\epsilon}{\rho} + O(\epsilon^2)\right)\right] = \frac{1}{\rho} = \frac{1}{1-r^2}.
 \end{equation}
 
Towards the derivation of an effective normalization scheme for the forward pass, we present a simplified derivation of the $1/\rho$ gain formula for the one-dimensional setting under white-noise input\footnote{We use the random input assumption for our normalization scheme as we found it to work well in practice.}: let $\Lambda =\lambda\in\Cmp$, and $ B=1$. Let $ p^*$ denote the conjugate of $p\in\Cmp$, we have that $|p|^2 = p^* p$ and in expectation over the input, using Eq.\eqref{eq:lin_rnn_unroll} and the fact that $\E[u_{k-i} u_{k-j}] = 0$ for $i\ne j$:
 \begin{align}
    \E|x_k|^2= \left(\sum_{i=0}^{k-1} \lambda^i \E[u_{k-i}]\right)\left(\sum_{j=0}^{k-1} \lambda^j \E[u_{k-j}]\right)^*
     = \sum_{i,j=0}^{k-1} \lambda^i  (\lambda^j)^* \E[u_{k-i} u_{k-j}]
    = \sum_{i=0}^{k-1}|\lambda|^{2i} \overset{\infty}{\to} \frac{1}{1-|\lambda|^{2}}.
    \label{eq:gamma_derivation}
\end{align}

\begin{table*}
\vspace{-3mm}
\begin{center}
\begin{footnotesize}
\begin{sc}
\setlength{\tabcolsep}{4pt}
\begin{tabular}{|l||*{6}{c|}}
\hline
& \textbf{sCIFAR} & \textbf{ListOps} & \textbf{Text} &\textbf{Retrieval} & \textbf{Pathfinder} & \textbf{PathX} \TBstrut\\
\hline
\textbf{LRU} & 89.0 (0.1) & 60.2 (0.8) & 89.4 (0.1) & 89.9 (0.1) & 95.1 (0.1) & 94.2 (0.4) \TBstrut\\
\hline
S4D (our reprod.) & 91.5 (0.2) & 60.2 (0.3) & 86.4 (0.0) & 89.5 (0.0) & 94.2 (0.3) & 97.5 (0.0) \TBstrut\\
S5 (our reprod.) & 88.8 (0.1) & 58.5 (0.3) & 86.2 (0.1) & 88.9 (0.0) & 95.7 (0.1) & 96.0 (0.1) \TBstrut\\
\hline
S4 (paper results) & \it{91.1} &\it{59.6}  & \it{86.8} & \it{90.9} & \it{94.2} & \it{96.4} \TBstrut\\
S4D-LegS (paper results) & 89.9 & 60.5 & 86.2 & 89.5 & 93.1 & 91.9 \TBstrut\\
S5 (paper results) & \it{90.1} & \it{62.2}  & \it{89.3} & \it{91.4} & 95.3 & \it{98.6} \TBstrut\\
\hline
\end{tabular}
\end{sc}
\end{footnotesize}
\end{center}
\vskip -0.1in
\vspace{-1mm}
\caption{\textit{Performance after adding the $\gamma$ normalization to the diagonal RNN with stable exponential parameterization and initialization on the ring (see \S\ref{sec:norm_pathx}). For PathX, we additionally use a smaller eigenvalue phase at initialization. We name this architecture \textbf{LRU}. We sweep $r_{\min}$ and $r_{\max}$ for setting the initialization distribution and the learning rate. We also report results from S4/S4D/S5 (along with reproductions in our own pipeline with similar hyperparameter sweeps as our RNN models). LRU reaches similar performance as these deep SSMs on all LRA tasks.
}}
\vspace{-1mm}
\label{tb:effect_normalization}
\end{table*}

Since the formula above holds for every Euclidean direction in our recurrence ($\Lambda$ is diagonal), we can add a normalization parameter that is initialized element-wise. Additionally, note that as $\lambda$ approaches 1, $1-|\lambda|^2$ approaches 0, making further adaptations with SGD of this parameter hard. Therefore, we use normalization parameter $\gamma^{\log}\in\R^N$, initialized element-wise as $\gamma_i^{\log}\leftarrow \log(\sqrt{1-|\lambda_i|^2})$,\footnote{We also tried setting $\gamma_i$ to $\sqrt{1-|\lambda_i|^2}$ in each training iteration, and found it to work similarly in practice to a trainable $\gamma$.} and modify the recurrence as:
\begin{equation}
     x_{k} = \Lambda x_{k-1} +\exp(\gamma^{\log})\odot (B u_{k}),
     \label{eq:normalized_rnn}
 \end{equation}
where $\odot$ denotes the element-wise product. The $\gamma$ parameter allows the RNN to adaptively scale the input fed into the corresponding eigendirection. 
We found the $\gamma$ normalization to consistently improve performance on tasks that benefit from initializing close to the unit disk, such as sCIFAR and Pathfinder, as shown in Tb.\ref{tb:effect_normalization}.

\vspace{-3mm}
\paragraph{Reducing Eigenvalue Phase at Initialization.}
In the context of the diagonalized recurrence, we have $\Lambda = \text{diag}(\exp(-\exp(\nu^{\log}) + \theta))$, where $\nu^{\log}\in\R^N$ is the vector of log eigenvalue magnitudes and $\theta\in\R^N$ the vector of eigenvalue \textit{phases}. While $\nu^{\log}$ encodes the distance to the origin, $\theta$ is the angle from the vector $1+0i$. \textit{For long sequences}, initializing uniformly $\theta\sim[0,2\pi]$ implies that most state entries will exhibit an overall large number of oscillations at initialization, see upper panel in Fig.\ref{fig:pathx_kernel}. Equivalently, in this setting, most state dimensions are the result of \textit{convolutions}\footnote{See \citep{gu2022parameterization} for a discussion of kernel perspectives.} capturing an average of \textit{local oscillation patterns}. This behavior is independent from the ability of capturing long-range dependencies~(controlled by $\nu^{\log}$), but pertains to the nature of the information stored by the RNN. 
\begin{figure}
\centering
\vspace{-2mm}
\includegraphics[width=1.0\textwidth]{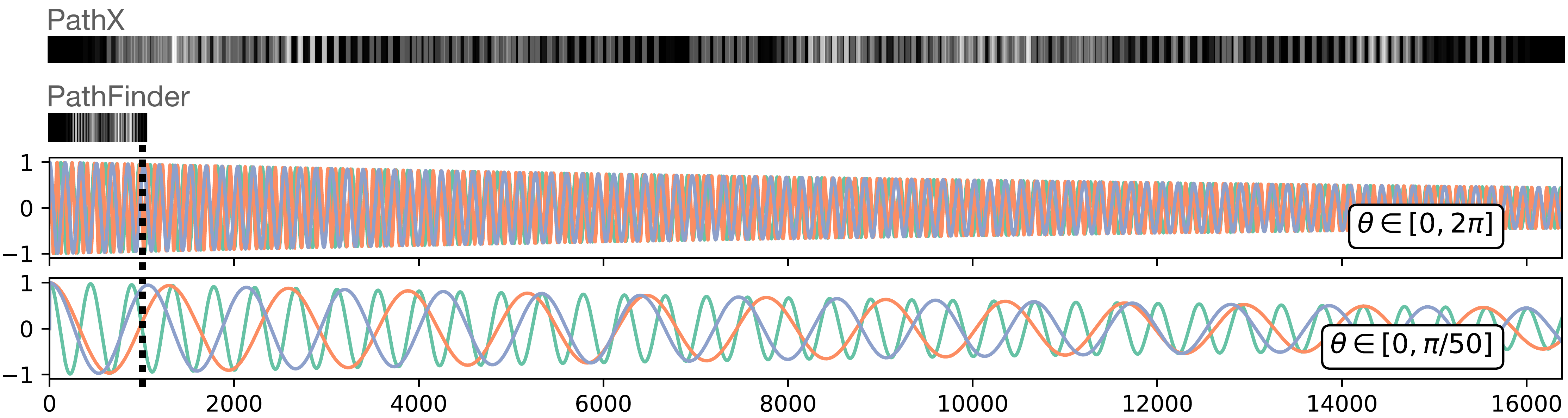}
    \vspace{-3mm}
    \caption{\textit{Evolution of $x\in\R^3$ under impulse input $u = (1, 0, 0,\dots, 0)\in\R^{16k}$. Plotted in different colors are the 3 components of $x$. $\Lambda$ has parameters $\nu_j=0.00005$ and $\theta_j$ sampled uniformly in $[0,2\pi]$ or with small phase $[0,\pi/50]$. For small sequences, such as $L=1024$~(PathFinder, sCIFAR), $[0,2\pi]$ produces kernels with acceptable overall number of oscillations: information about $u_0$ is recalled only a few times in the overall state history. Instead, for high $L$, the range of the imaginary part at initialization has to be smaller to obtain a similar effect.}}
    \label{fig:pathx_kernel}
\end{figure}
\begin{figure}
    \centering
    \includegraphics[height=0.22\textwidth]{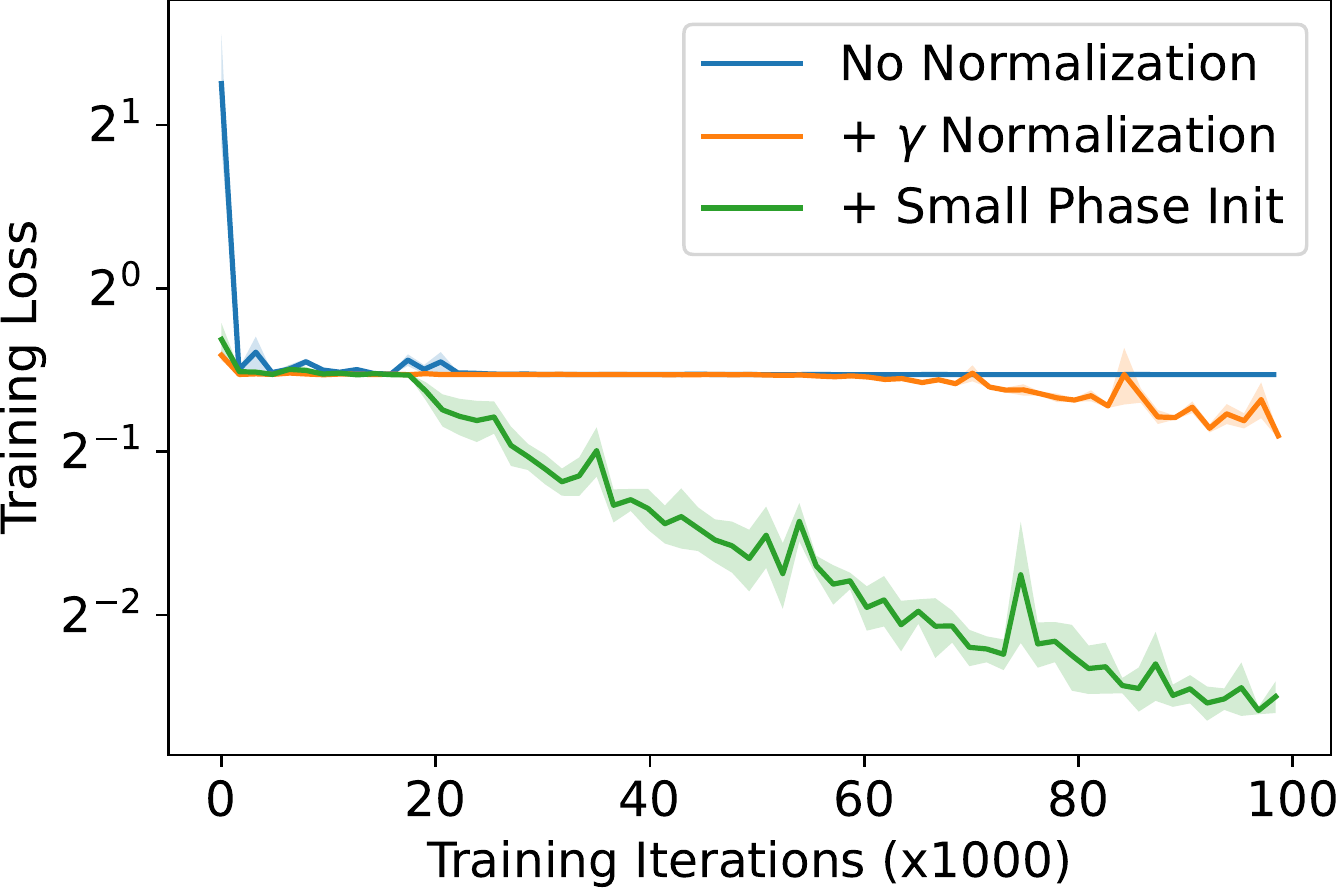}
    \includegraphics[height=0.22\textwidth]{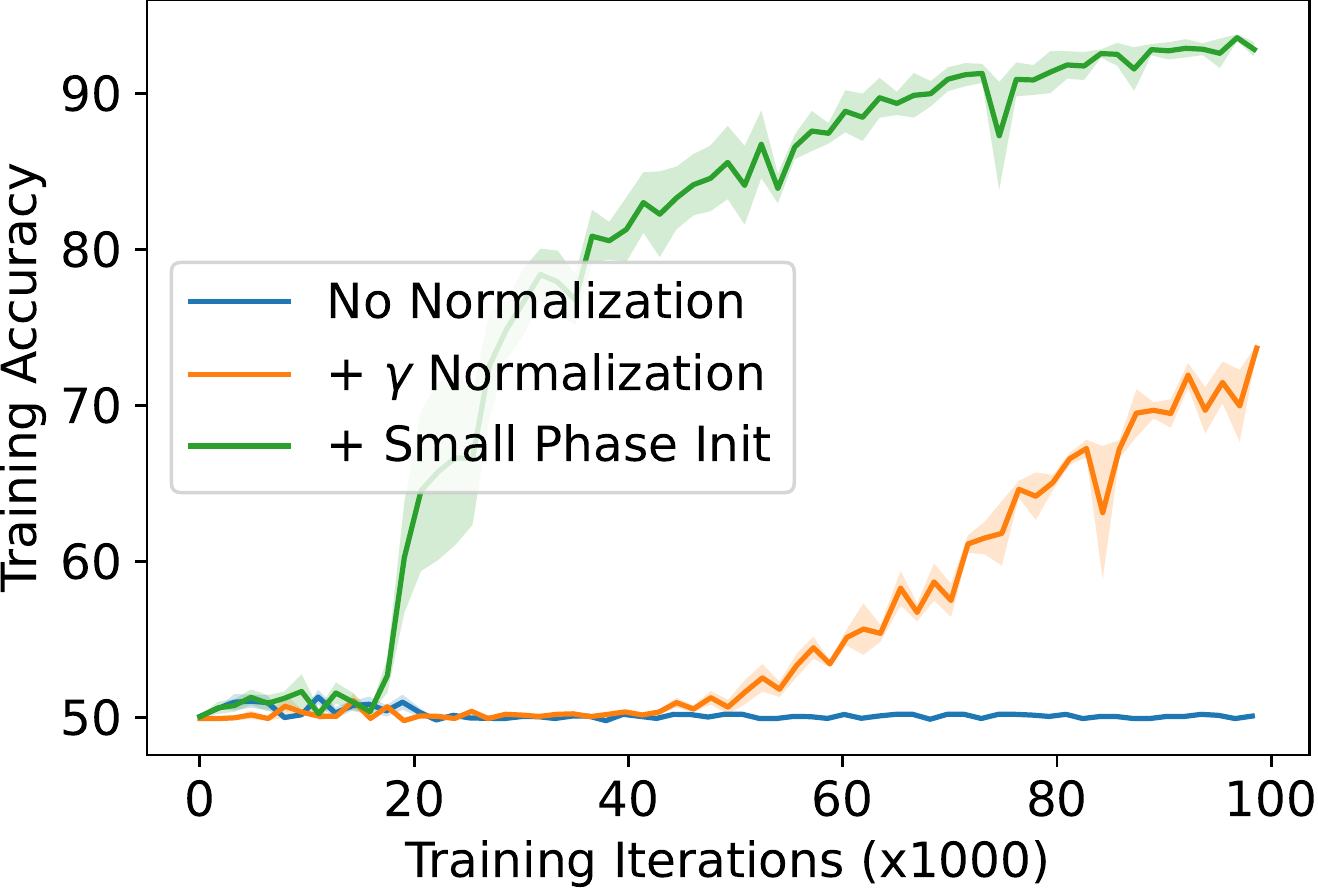}
    \includegraphics[height=0.22\textwidth]{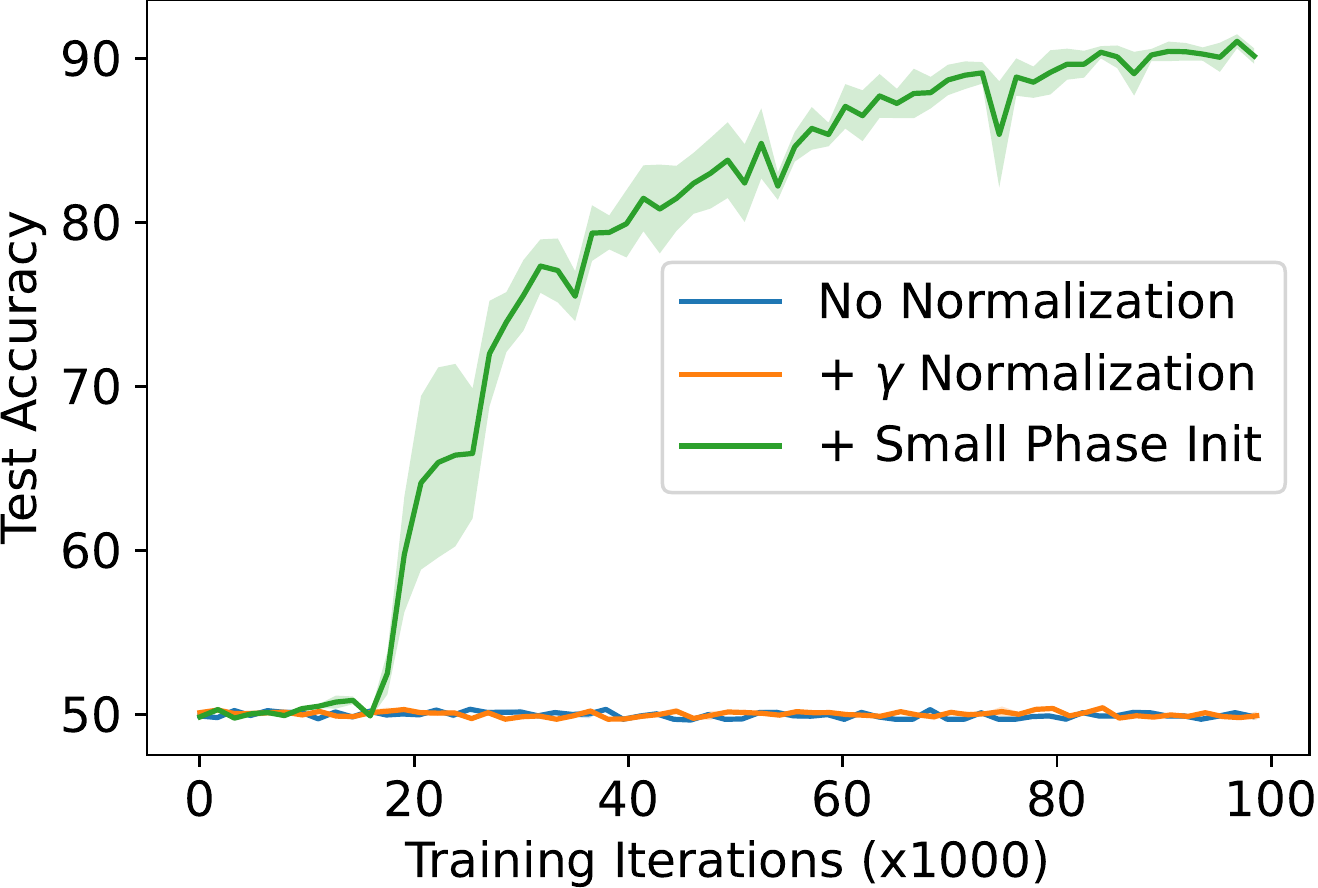}
        \vspace{-3mm}
    \caption{\textit{Effect of normalization and using a small phase at initialization on the PathX task. 
    For each setting, we show mean and standard errors over three independent runs for 100k iterations. 
    Without normalization, the model presents higher loss values at initialization and quickly converges to a suboptimal value, where train and test accuracy are both at random chance. Adding normalization helps: the train loss is lower at initialization, and the optimizer is able to escape the suboptimal region and train accuracy also increases. Interestingly, this model still fails to generalize at all. Finally, 
    reducing initialization phase~(i.e.~tuning the range of $\theta$) dramatically improves convergence on the training set, while also generalizing to the test set.}}
    \label{fig:normalization_smallphase_pathx}
    \vspace{-3mm}
\end{figure}
Therefore, we claim that initializing $\Lambda$ with uniform phase on long sequence data inherently biases the network towards learning spurious features in the input sequence. The model cannot recover from this suboptimal initialization: we indeed observe that, for our best to far model on PathX, the training loss after a few iterations converges to a highly suboptimal minimizer which leads to random chance test performance~(see Fig.\ref{fig:normalization_smallphase_pathx}). To fix this issue, we found it sufficient to restrict the range of $\theta$ to a thin slice around $0$, biasing the model towards learning more global features. 
Since the optimal values of $\theta$ are small, we parameterize the phase logarithmically: $\theta = \exp (\theta^{\log})$, where $\theta^{\log}$ is optimized, to aid optimization.

Restricting the range of the phase at initialization to be $[0, \pi/10]$, our LRU achieved \textit{$94.2\%$ on PathX}, aligning with state-of-the-art deep SSMs. We did not explore using a smaller phase at initialization for the other LRA tasks, although we believe this might further improve performance on other tasks as well. Note that using both $\gamma$ normalization and restricting the eigenvalue phase at initialization were crucial to solving PathX. We were unable to learn when using restricted phase at initialization without also introducing $\gamma$ normalization.

With all the components of \S\ref{sec:how_to} taken together, we name this new model the \textbf{Linear Recurrent Unit}~(or \textbf{LRU} for short). It provides a flexible, interpretable, and principled framework for initializing and learning deep RNNs efficiently, and matches performance and efficiency of deep SSMs across all LRA tasks as shown in Tb.\ref{tb:effect_normalization}.

\section{Insights on S4 and Variants}
\label{sec:explaining_s4}
We believe our ablations in \S\ref{sec:how_to} explain the underlying mechanisms driving the success of deep SSMs. Hence, to conclude the paper, in this section, we inspect in detail the main similarities and differences between our LRU model and diagonal SSMs, and elaborate a few insights. As in \S\ref{sec:preliminaries}, to avoid technicalities, we provide a simplified discussion capturing the main features of models stemming from the original S4 paper. For a comparison of different models, we defer the reader to \S\ref{sec:related}.

As detailed in \S\ref{sec:preliminaries}, diagonal SSMs~(DSS, S4D, S5) are instantiated and parameterized through \textit{discretization} of a latent continuous-time model $\dot x_{\text{ct}}(t)  = \tilde Ax_{\text{ct}}(t) + \tilde B u_{\text{ct}}(t)$, where $A = \diag(\tilde a)$ is initialized with complex entries, often prescribed or inspired by HiPPO theory~\citep{gu2020hippo}. Zero-Order-Hold~(ZOH) discretization with stepsize $\Delta$ leads to the recurrence $x_k = \exp(\Delta \tilde A) x_{k-1} +(\exp(\Delta \tilde A)-I)\tilde A^{-1}\tilde B u_k$. This formula, while arguably complex compared to our Eq.\eqref{eq:normalized_rnn}, relates to it as outlined in the next paragraphs.

\vspace{-3mm}

\paragraph{Matrix exponentials make training easier.}  The exponential in the ZOH formula is due to exact integration of $\dot x_{\text{ct}}(t) = \tilde Ax_{\text{ct}}(t)$, which leads to $x_{\text{ct}}(\Delta k) = \exp(\Delta \tilde A)x_{\text{ct}}(\Delta (k-1))$. In addition, to enforce stability, in models inspired by S4 the real part of $A$ is often fed into a positive nonlinearity, as we also do in \S\ref{sec:exponential}. From our results \S\ref{sec:exponential} and our discussion on optimization advantages~(see also \S\ref{app:opt}), we claim that the power of exponential parameterization is not necessarily attributable to accurate integration~(which is not present in our system), but is more fundamentally rooted in a magnitude-phase decoupling on the recurrence~(this makes training with Adam easier, see Fig.\ref{fig:learning_powers}), as well as in the overall advantage of learning in diagonalized space~(see Tb.\ref{tb:effect_parametrization_exp}). We also note that stabilizing the recurrence by adding a nonlinearity was beneficial also in our experiments, although this is not prescribed by the theory underlying S4.

\vspace{-3mm}

\paragraph{Structured initialization is not necessary.} While~\citet{gu2022parameterization,gupta2022simplifying, smith2022simplified} also discuss initializations for $A$ deviating from the HiPPO structure~(see~\S\ref{sec:preliminaries} and \S\ref{sec:related}), to the best of our knowledge we are the first to show that simple uniform initialization on a slice of the unit disk, combined with proper normalization, is able to also solve the hardest task in LRA: PathX.\footnote{Among the models in~\citep{gu2022parameterization}, only S4D-inv and S4D-LegS~(options heavily inspired by the HiPPO theory) perform beyond random guessing on PathX. In S5, the skew-symmetric component of the HiPPO matrix is used for initialization.} We also show~(Tb.\ref{tb:effect_parametrization_exp}) that uniform initialization on the disk, which is simply the diagonalized version of Glorot initialization~(Thm.~\ref{thm:ginibre}), is sufficient to achieve performance close to more complex deep state-space models on the remaining LRA tasks. Our results ultimately suggest that HiPPO theory, while fundamental for the development of this field, should not be thought of as the main source of S4 success.

\vspace{-3mm}
\paragraph{Discretization changes initialization spectrum.} For simplicity, let us restrict our attention to S4D-Lin, for which $A = \diag(\tilde a)$ with $\tilde a_n = -\frac{1}{2}+ i\pi n$, yielding a diagonal transition matrix with elements~(i.e. eigenvalues) initialized at $\exp(-\Delta/2 + i\pi \Delta n)$. Under typical choices e.g. $\Delta = 1e{-3}, N = 128$, the SSM eigenvalues have magnitude $\exp(-\Delta/2)\approx 0.9995$, and phase $\theta = \pi \Delta n\overset{\sim}{\in}[0, \pi/8]$ --- i.e. initialization is performed on a ring\footnote{For all diagonal SSMs, $\Delta$ is actually a vector initialized in the range $[\Delta_{\min},\Delta_{\max}]$. This interval can be directly mapped through the exponential map to a ring in complex space~(see Lemma~\ref{lemma:sampling_exp}).} close to the unit circle in $\Cmp$, with restricted phase connected to the eigenvalues magnitude. As is clear from the results in \S\ref{sec:exponential} and \S\ref{sec:norm_pathx}, linking the eigenvalues phase and magnitude is not necessary to achieve good performance: indeed, as it can be seen in Tb.\ref{tb:effect_normalization}, test accuracy on the Long Range Arena~(except PathX) can be recovered by using a more natural magnitude-independent initialization on the complete ring. As we discussed in ~\S\ref{sec:norm_pathx}, changing the initialization phase to a small range around $0$ can be motivated by first principles, yet is only needed for extremely long sequences: this modification is already hard-coded in S4, where choosing a small $\Delta$ also shrinks the phase.\footnote{This is a useful effect of having a latent continuous-time model: choosing eigenvalues close to the unit circle~(i.e. small $\Delta$) changes the oscillation frequencies in the discretized system.} However, our results clearly show that connecting real and imaginary parts during training through the $\Delta$ parameter is not necessary to achieve good performance, even on PathX.

\vspace{-3mm}

\paragraph{Discretization performs normalization.} The most striking visual difference between ours and ZOH-discretized S4 recurrence is in the matrix multiplier for $u_k$: $(\exp(\Delta \tilde A)-I)\tilde A^{-1}\tilde B$. After conducting experiments on S4D, we found that simply replacing this multiplier with its first-order expansion in $\Delta$, i.e. $\Delta \tilde B$, yields a close match in performance. For input dimension $H=1$ and unit $B\in\R^{N\times 1}$~(to keep reasoning simple), the corresponding recurrence is $x_k = \exp(\Delta \tilde a) + \Delta 1_{N} u_k$. Elementwise unrolling of this recurrence -- without the $\Delta$ in front of $u$ -- yields $|x_{k,i}| \le \sum_{j=0}^{k-1}|\exp(\Delta \tilde a_i)|^j u_{k-j,i}$, which in the limit $k\to\infty$ gives $O(\Delta^{-1})$. Therefore, the $\Delta$ multiplier in front of $B$ effectively scales the recurrence to avoid blow-ups --- similar to our $\gamma$ normalization factor.
\vspace{-4mm}

\paragraph{Parameter sharing is not necessary.} As a result of discretization, the $\Delta$ parameter multiplying both $\tilde A$ and $\tilde B$ couples the recurrence formula with the input projection during training. In our S4 ablations, we found that decoupling these in two separate parameters --- keeping the same initialization to guarantee no blow-ups~(see last paragraph) --- does not decrease performance, suggesting that the ODE discretization viewpoint (which induces parameter sharing) is not necessary to achieve S4 performance.

From this discussion, we conclude that the success of (diagonal) state-space models is attributable to the use of linear recurrences and  complex diagonal exponential matrices, combined with the normalization and initialization induced by discretization.
On the other hand, other artifacts of discretization such as parameter sharing or the continuous-time interpretation do not necessarily contribute to its performance.

\vspace{-2mm}
\section{Conclusion}
In this paper, we introduce a new RNN layer called the Linear Recurrent Unit or LRU and show how it can be effectively and efficiently used as core layers of deep sequence models.
We provide theoretical insights and extensive ablations on a series of step-by-step modifications of a vanilla RNN---linearization, diagonalization, stable exponential parameterization and normalization---that substantially improve performance, especially on tasks requiring long range reasoning.
While our recurrence shares similarities with modern deep SSMs, our design does not rely on discretization of a latent continous-time system or on structured transition matrices. Instead our improvements directly follow from initialization and forward pass analysis arguments standard in the deep learning community, starting from a Glorot-initialized RNNs. Our final model matches the performance of modern deep state-space models~(e.g. S4 or S5) on all LRA tasks. 

\section*{Acknowledgements}
The authors would like to thank Michalis Titsias, Aleksandar Botev, James Martens and Yee Whye Teh for the interesting discussions and perspectives on our work.


\bibliographystyle{abbrvnat}







\newpage
\bibliography{main}

\newpage
\section*{Supplementary Materials}
\appendix

\section{Simplified Implementation of the Linear Recurrent Unit}
\label{app:pseudocode}
We present here a simplified JAX implementation \citep{bradbury2018jax} of the Linear Recurrent Unit (LRU). The state of the LRU is driven by the input $(u_k)_{k=1}^L$ of sequence length $L$ according to the following formula (and efficiently parallelized using an associative scan): $x_{k} = \Lambda x_{k-1} +\exp(\gamma^{\log})\odot (B u_{k})$, and the output is computed at each timestamp $k$ as follows: $y_k = C x_k + D u_k$. In our code, $B,C$ follow Glorot initialization, with $B$ scaled additionally by a factor 2 to account for halving the state variance by taking the real part of the output projection.  $D$ is random $H$-dimensional and mutiplies elementwise each $u_k$, where $k$ is the timestamp. $\Lambda$ is initialized with the help of Lemma~\ref{lemma:sampling_exp}, with phase potentially restricted to a thin slice~(see \S\ref{sec:norm_pathx}). 
\begin{center}
\includegraphics[width=\textwidth]{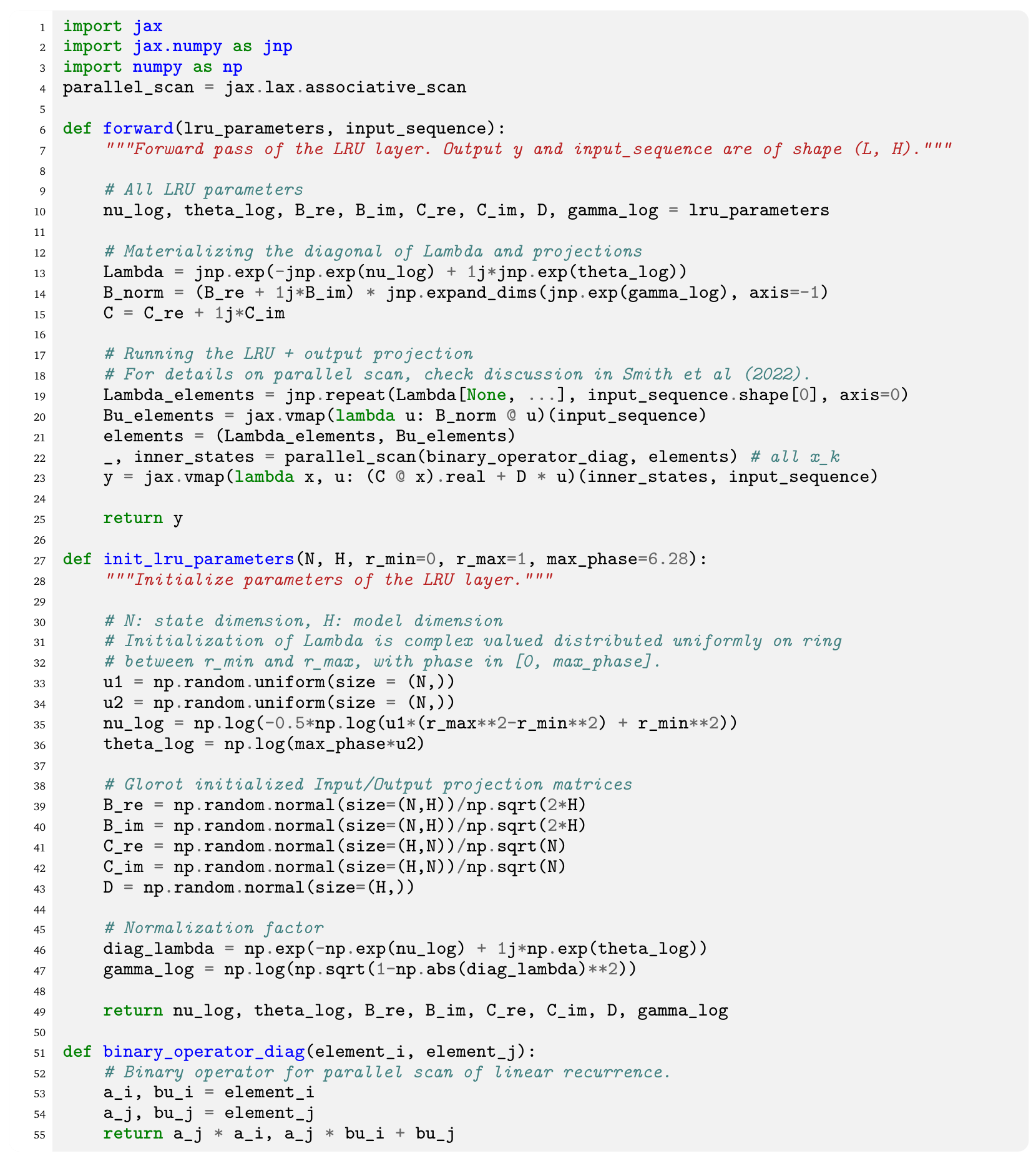}
\end{center}

\section{Related works}
\label{sec:related}

We first discuss standard RNN-based approaches for sequence-to-sequence modeling, and then provide a historical overview on the progress of the literature stemming from the S4 paper~\citep{gu2021efficiently}.

\vspace{-2mm}

\paragraph{Recurrent neural networks (RNNs).} Before the rise of transformers~\citep{vaswani2017attention}, RNNs were widely used in various applications of natural language processing tasks such as language modeling \citep{mikolov2010recurrent}, machine translation \citep{cho2014learning} and text summarization \citep{nallapati2016abstractive}. The modern RNN structure~(see Eq.\ref{eq:RNN}) is mainly attributed to the works of \citet{rumelhart1985learning}. However, it is possible to see the Hopfield Networks as a particular form of RNN \citep{hopfield1982neural}. Modern RNN formulations are also often related to the Elman Networks \citep{elman1990finding}. The issue of vanishing or exploding gradients, as described by \citet{bengio1994learning, pascanu2013difficulty}, is one barrier to training Recurrent Neural Networks~(RNNs) with gradient descent. This problem limits the ability of RNNs to learn, especially on tasks with long input sequences. One of the critical contributions to the success of RNNs was the introduction of gating mechanisms such as the Long Short-Term Memory (LSTM) proposed by the \citet{hochreiter1997long}. LSTMs address the vanishing gradients problem by introducing input, output, and forget gates, which enable the network to selectively remember or forget information from previous time steps. Another popular variant of gated RNNs is the Gated Recurrent Unit~(GRU) \citep{cho2014learning} which simplifies the LSTM architecture by merging input and forget gates into a single update gate. 

\vspace{-2mm}

\paragraph{Mitigating the vanishing gradient problem with orthogonal and unitary RNNs.}  Recently, \citet{arjovsky2016unitary} introduced unitary evolution RNNs (uRNN), where eigenvalues in the RNN transition matrix~(see Eq.~\eqref{eq:RNN}) are restricted to live on the unit circle. The induced map driving the hidden state evolution, therefore, mixes state components taking into account new inputs --- but the signal from past timestamps is not exponentially vanishing/exploding as in the vanilla RNN case~(see discussion on stability in \S\ref{sec:eig}). This idea is powerful but introduces two problems: (1) choosing unitary transitions restricts the function approximation class, and (2) training unitary matrices is expensive since a projection on the Stiefel manifold is required at each gradient step. To resolve the second issue, many works devoted attention to carefully designed reparameterization of the transition matrix as e.g., with the product of simpler matrices~\citep{arjovsky2016unitary}, Givens rotations~\citep{jing2017tunable}, Householder reflections~\citep{mhammedi2017efficient}, or as exponentials of skew-symmetric matrices~\citep{hyland2017learning, lezcano2019cheap}. The approximation capacity of these models is discussed and improved in~\citep{wisdom2016full}. A further step in designing efficient orthogonal RNNs is provided by~\citet{helfrich2018orthogonal}, who parametrized skew-symmetric matrix using the Cayley transforms, resulting in a fully real parameter space. Other works which proposed conceptually different solutions to mitigate the vanishing gradient problem include combinations with rectified linear units~\citep{le2015simple}, Lipschitz RNNs~\citep{erichson2021lipschitz}, and approaches based on dilated convolutions to increase context size~\citep{oord2016wavenet, bai2018empirical}

\vspace{-2mm}

\paragraph{Deep state-space models~(SSMs), a historical overview.} Inspired by interesting approaches involving continuous-time representation for recurrent neural networks~\citep{voelker2019legendre}, \citet{gu2020hippo} recently provided an alternative view on the vanishing gradient problem: one can design \textit{linear} continuous-time state-space models~(SSMs), of the form $\dot x(t) = Ax(t) + Bu(t)$ where the state $x(t)\in\R^N$ is guaranteed to compress all relevant~(under a certain metric) information about previously observed (one-dimensional) inputs $u([0,t])$. For instance, by using specific pair of matrices $(A\in\R^{N\times N},B\in\R^{N\times 1})$, one can discretize the continuous-time SSM above using a stable, accurate integrator~(e.g., bilinear or zero-order-hold) and retrieve the hidden state $x(t)$, which contains the coefficients for the best $N$-th degree polynomial approximation to $u([0,t])$. The idea of \citet{gu2020hippo} was to then use the resulting discretized \textit{structured}~(i.e., using structured HiPPO matrices) state-space model as a starting for the design and initialization of a novel gated RNN.

Later, \citet{gu2021efficiently} scaled up this idea into a deep architecture, where a collection~(one for each input dimension) of discretized continuous-time structured SSM was placed at each layer as a substitute\footnote{This idea is also leveraged in FNet~\citep{lee2021fnet}, where the attention mechanism is replaced with a simpler linear token-mixing strategy.} for the attention block, in an attempt to mitigate the $O(L^2)$ issue in transformers and provide a theoretically principled component for sequence-to-sequence modeling. The model reached state-of-the-art on the Long Range Arena benchmark~\citep{tay2020long}, effectively showcasing the power of discretized linear recurrences using structured transition matrices. Notably, the resulting model, named \textbf{S4}, uses a convenient and stable representation of the HiPPO transition, which is initialized using a normal + low-rank matrix and then learned efficiently in diagonal + low-rank form using fast Fourier transforms~(FFTs) and Cauchy kernels.

In the months following the publication of S4, \citet{gupta2022diagonal} noticed that most of S4 performance can be retrieved by only considering the diagonal component of the HiPPO matrix, and therefore showed the power of discretized diagonal structured continuous-time state space models. This architecture is known as \textbf{DSS}. As the interest of the community was rising, with first applications of DSS and S4 in language~\citep{mehta2022long}, vision~\citep{nguyen2022s4nd} and audio~\citep{goel2022s}, \citet{gu2022parameterization} further simplified DSS providing a diagonal form~(\textbf{S4D}) with theoretical guarantees in the infinite width setting. Notably~\citet{gu2022parameterization} showed that, to retrieve most performance of S4, one can simply initialize the transition matrix $A$ in diagonal form, with entries $a_n = -\frac{1}{2}+ i\pi n$ (S4D-Lin) or
$a_n = -\frac{1}{2}+ i\frac{N}{\pi}\left(\frac{N}{n+1}-1\right)$~(S4D-Inv). Our interest in S4-like models spiked at this point since the findings of~\citet{gu2022parameterization} suggest that, given the effectiveness of such simplified versions of $A$, the root of S4 success might be attributable to more fundamental effects are orthogonal to the HiPPO theory.

Shortly after,~\citet{smith2022simplified} found that one can also depart from the formal one-dimensional discretization structure of S4, rooted in the HiPPO theory, and considered a simplified version where all input dimensions are efficiently and simultaneously processed using parallel scans~\citep{martin2017parallelizing} --- not separately like in S4, S4D, and DSS. This model~(named \textbf{S5}) set a new state-of-the art on PathX, the hardest task in the Long Range Arena, and provides further evidence for a conceptually simpler motivation for the performance of deep state-space models. Indeed, as already mentioned, S5 is not precisely the discretization of a latent continuous-time SSM, yet still includes parameters like discretization stepsizes that have an ambiguous interpretation in this context\footnote{One can still view S5 as a discretized version of a continuous-time SSM. However, this requires adjusting the input projection matrix.}, suggesting further investigations are needed.

At the same time, a few interesting works developed novel variants of the S4 architecture. \textbf{Liquid S4} used the original (non-diagonal) S4 formulation combined with liquid time-constant networks~\citep{hasani2021liquid,hasani2022liquid}.
Similar to DSS, S4D, and S5, \textbf{Mega} also simplified S4 to a diagonal SSM \citep{ma2022mega} while showing additionally that restricting the diagonal $A$ to real numbers -- giving it an exponential moving average (EMA) interpretation -- can still work well when combined with attention and a gated block design.
Another intriguing view was provided by the \textbf{SGConv} model~\citep{li2022makes}, which leverages the convolutional interpretation of SSMs~\citep{gu2021combining} to design a purely filter-based version of S4, with no latent continuous-time model or need for discretization.

The discretization viewpoint also attracted the interest of~\citet{gupta2022simplifying}, concurrent to this work, who pointed out that, after numerical integration, diagonal state-space models and linear RNNs share the same function approximation class. \citet{gupta2022simplifying} then introduced \textbf{DLR}, most closely related to DSS and S4D (each input is processed independently at each layer) but where the discretization stepsize~$\Delta$ is absorbed into the continuous-time transition matrix $A$~(see \S\ref{sec:preliminaries}).
Their focus was on a new set of synthetic long-range tasks with strong supervision~(e.g. segmentation), while ours is on the established Long Range Arena benchmark.

To conclude, we point the reader to interesting recent applications of models inspired by the S4 architecture. In addition to earlier applications in NLP~\citep{mehta2022long}, more sophisticated architectures based on S4 recently showed great promise in language modeling~\citep{dao2022hungry, ma2022mega}. Specifically, \citet{dao2022hungry} designed a new generative language model, \textbf{H3}, that outperforms GPT-Neo-2.7B with SSMs, augmented with two attention layers. Besides language, deep state-space models were also found successful for long video/audio understanding and generation tasks~\citep{islam2022long,nguyen2022s4nd, goel2022s}, and have attracted interest in biology~\citep{bordin2022novel} and time series forecasting~\citep{zhou2022film}.

\clearpage

\section{Additional experimental results}
\label{app:additional_results}

\subsection{Training speedups}
\label{app:training_speedups}
In Tb.\ref{tb:steps_per_sec}, we show training speed comparisons of the LRU with a regular RNN with tanh activations, as well as with the S4D and S5 models. As we elaborate in \S\ref{sec:exp_setup}, for the LRU, we closely followed the optimal model sizes of the S5 model. Consequently, we also see similar training speeds as the S5 model on all tasks.

\begin{table*}[ht]
\begin{center}
\begin{scriptsize}
\begin{sc}
\begin{tabular}{|l||*{6}{c|}}
\hline
Model & sCIFAR & ListOps & Text & Retrieval & Pathfinder & PathX \TBstrut\\ \hline
Tanh RNN & 2.0       & 1.1        & 0.5        & 0.5  & 2.1 & 0.14 \TBstrut\\ 
LRU       & 15.9 (8x) & 2.1 (1.9x) & 14.7 (29x) & 5.7 (11.4x) & 15.5 (7.4x)  & 2.4 (17x)  \TBstrut\\ \hline
S4D (our reproduction) & 13.5 & 2.2 & 10.6 & 3.0 & 24.5 & 2.6 \TBstrut\\
S5  (our reproduction) & 15.9 & 2.2 & 14.4 & 5.7 & 15.6 & 2.3 \TBstrut\\ \hline
\end{tabular}
\end{sc}
\end{scriptsize}
\end{center}
\vskip -0.1in
\caption{\textit{Speeds (steps/sec) during training on a A100 GPU. We also show the speedup of the LRU over the tanh RNN for each task. The batch size used for each task is specified in Tb.\ref{tb:hparams}.}}
\label{tb:steps_per_sec}
\end{table*}

\subsection{Effect of stability and normalization}

In this section, we explore further the effect of introducing stability during training (\S\ref{sec:exponential}), as well as introducing the $\gamma$ normalization factor as shown in Eq.\eqref{eq:normalized_rnn}. To do this, we consider the sCIFAR experiment where we sweep over different settings of $r_{\max}$ and $r_{\min}$ to see the effect when initializing closer to the unit disk. We keep the learning rate fixed at 0.004 for these experiments, which we found to be optimal when initializing with $r_{\max}=1.0$ and $r_{\min}=0.0$ under a stable exponential parameterization.

We show our results in Tb.\ref{tb:effect_stability_gamma_cifar}. In the first table Tb.\ref{tb:effect_stability_gamma_cifar}(A), we show results with our baseline where we use the exponential parameterization described in \S\ref{sec:exponential}. We see that under this setting, the optimal performance is achieved when $r_{\max} = r_{\min} - 0.9$, and performance degrades as $r_{\max}$ is increased beyond 0.9.

In Tb.\ref{tb:effect_stability_gamma_cifar}(B) we show results after enforcing stability. We now notice that for each $r_{\min}$, the optimal performance is achieved by a higher $r_{\max}$ than before, i.e., training is more when initializing closer to the unit disk. Our optimal performance in this setting is achieved using $r_{\min} = 0.0$ and $r_{\max} = 0.99$. Note that even in this setting, performance can sometimes degrade when moving to even higher $r_{\max}$.

Finally, in Tb.\ref{tb:effect_stability_gamma_cifar}(C) we also incorporate the $\gamma$ normalization factor, and we now notice no degradation in performance even when $r_{\max} = 0.999$. We found training to be more stable in this setting, and our best result of 89.0\% performance is also obtained in this setting, with $r_{\min}=0.9$ and $r_{\max}=0.999$.

These ablations further motivate the benefits of enforcing stability and using the normalization parameter for better performance and more stable training, particularly when required to learn very long-range dependencies.

\begin{table}[ht]
\begin{center}
\begin{scriptsize}
\begin{sc}
        \begin{tabular}{|l||*{3}{c|}}
        \hline
          \diagbox{$r_{\max}$}{$r_{\min}$}    & $0$ & $0.5$ & $0.9$ \TBstrut\\
        \hline
        0.9 & \textbf{87.6} (0.4) & \textbf{87.8} (0.1)  & \textbf{87.9} (0.2) \TBstrut\\
        0.99 &   83.8 (0.9) & 85.8 (1.2) &  81.9 (3.8) \TBstrut\\
        0.999 &  83.9 (0.2) & 84.8 (0.4)  & 84.8 (0.8) \TBstrut\\
        \hline
        \end{tabular} \\ \vspace{1mm}
        (a) No stability.\ \\ \vspace{3mm}

        \begin{tabular}{|l||*{3}{c|}}
        \hline
          \diagbox{$r_{\max}$}{$r_{\min}$}    & $0$ & $0.5$ & $0.9$ \TBstrut\\
        \hline
        0.9 & 86.2 (0.2) & 86.6 (0.3)  & 87.3 (0.1) \TBstrut\\
        0.99 &  \textbf{87.8} (0.2) & \textbf{87.7} (0.1) & \textbf{88.1} (0.0) \TBstrut\\
        0.999 &  87.4 (0.2) & 87.4 (0.1)  & 87.5 (0.4) \TBstrut\\
        \hline
        \end{tabular} \\ \vspace{1mm}
        (b) With stability.\ \\ \vspace{3mm}
        
        \begin{tabular}{|l||*{3}{c|}}
        \hline
         \diagbox{$r_{\max}$}{$r_{\min}$}    & $0$ & $0.5$ & $0.9$ \TBstrut\\
        \hline
        0.9 & 86.4 (0.1) & 86.5 (0.1)  & 88.3 (0.1) \TBstrut\\
        0.99 &  \textbf{88.1} (0.1) & 88.4 (0.1) & \textbf{89.0} (0.2) \TBstrut\\
        0.999 &  \textbf{88.1} (0.1) & \textbf{88.6} (0.0)  & \textbf{89.0} (0.1) \TBstrut\\
        \hline
        \end{tabular} \\ \vspace{1mm}
        (c) With $\gamma$ normalization.
\end{sc}
\end{scriptsize}
\end{center}
\vspace{-3mm}
        \caption{\textit{Effect of stability and normalization and different $r_{\min}$ and $r_{\max}$ values on test accuracy for the sCIFAR10 task. Both stability and normalization allow for initializing eigenvalues closer to the unit disk, resulting in improved performance.}}
        \label{tb:effect_stability_gamma_cifar}
\end{table}

\subsection{Expanded tables}

Below we show our full results on the Long Range Arena, expanding on Tables~\ref{tb:effect_nonlinearity},~\ref{tb:effect_parametrization_exp}, and \ref{tb:effect_normalization} in the main paper. The tables are presented in logical order: in Table~\ref{tb:effect_nonlinearity_stddevs_app}, we show that vanilla~(dense) RNNs profit from dropping recurrent nonlinearities when used in the context of the architecture in Fig.~\ref{fig:deep_rnn}. Next, in Table~\ref{tb:effect_parametrization_exp_stddevs} we diagonalize our linear RNN model from \S\ref{sec:linear_rnns} and show how different parametrization for the diagonal elements affect performance. For all the rows in Table~\ref{tb:effect_parametrization_exp_stddevs}, initialization of the diagonal RNN was performed uniform on the disk, to match the random Glorot initialization of our dense version~(Thm.~\ref{thm:ginibre}).

Further, the last row in Table~\ref{tb:effect_parametrization_exp_stddevs} shows the positive effects of changing initialization distribution to a thin ring close to the circle boundary --- effectively enabling long-range reasoning through mitigation of vanishing gradients. Our settings for the ring are reported on the first row of Table~\ref{tb:effect_normalization_stddevs}. Finally, the second row of this table shows the improvements that can be achieved by including model normalization~(Eq.~\eqref{eq:normalized_rnn}), which closes the accuracy gap with deep SSMs.

\begin{table*}[ht]
\begin{center}
\begin{scriptsize}
\begin{sc}
\begin{tabular}{|l||*{6}{c|}}
\hline
Recurrence & \textbf{sCIFAR} & \textbf{ListOps} & \textbf{Text} &\textbf{Retrieval} & \textbf{Pathfinder} & \textbf{PathX} \TBstrut\\
\hline
RNN-Lin   & 72.2 (0.2) & 50.4 (0.2) & 89.1 (0.1)  & 89.1 (0.1)  & \XSolidBrush & \XSolidBrush \TBstrut\\
RNN-ReLU & 69.7 (0.2)  & 37.6 (8.0)  & 88.0 (0.1)  &  88.5 (0.1) & \XSolidBrush & \XSolidBrush \TBstrut\\
RNN-Tanh & 69.9 (0.3)  & 43.9 (0.1)  &  87.2 (0.1) & 88.9 (0.2)   & \XSolidBrush & \XSolidBrush \TBstrut\\
\hline
S4D (our reproduction) & 91.5 (0.2) & 60.2 (0.3) & 86.4 (0.0) & 89.5 (0.0) & 94.2 (0.3) & 97.5 (0.0) \TBstrut\\
S5 (our reproduction) & 88.8 (0.1) & 58.5 (0.3) & 86.2 (0.1) & 88.9 (0.0) & 95.7 (0.1) & 96.0 (0.1) \TBstrut\\
\hline
S4 (paper results)& \it{91.1} &\it{59.6}  & \it{86.8} & \it{90.9} & \it{94.2} & \it{96.4} \TBstrut\\
S4D-LegS (paper results) & 89.9 & 60.5 & 86.2 & 89.5 & 93.1 & 91.9 \TBstrut\\
S5 (paper results) & \it{90.1} & \it{62.2}  & \it{89.3} & \it{91.4} & 95.3 & \it{98.6} \TBstrut\\
\hline
\end{tabular}
\end{sc}
\end{scriptsize}
\end{center}
\vskip -0.1in
\caption{\textit{Placing a Vanilla RNN as recurrent core in the architecture of Fig.~\ref{fig:deep_rnn}. Shown is the effect of removing the RNN non-linearity on test accuracy (\S\ref{sec:linear_rnns}). }}
\vspace{-2mm}
\label{tb:effect_nonlinearity_stddevs_app}
\end{table*}

\begin{table*}[ht]

\setlength{\tabcolsep}{5pt}
\begin{center}
\begin{scriptsize}
\begin{sc}
\begin{tabular}{|l||*{6}{c|}}
\hline
 & \textbf{sCIFAR} & \textbf{ListOps} & \textbf{Text} &\textbf{Retrieval} & \textbf{Pathfinder} & \textbf{PathX} \TBstrut\\
\hline
Dense $A$ & 72.2 (0.2) & 50.4 (0.2) & 89.1 (0.1) & 89.1 (0.1) & \XSolidBrush & \XSolidBrush \TBstrut\\
$\Lambda$ Real + Im  & 86.5 (0.1) & 58.8 (0.3) & 87.4 (0.3) & 87.8 (0.5) & \XSolidBrush & \XSolidBrush \TBstrut\\
$\Lambda$ Exp & 85.4 (0.7) & 60.5 (0.3) & 86.5 (0.4) & 89.4 (0.1) & 65.4 (9.0) & \XSolidBrush \TBstrut\\
$\Lambda$ Stable Exp & 87.2 (0.4) & 59.4 (0.3) & 87.6 (0.3) & 89.1 (0.2) & 93.5 (0.5) & \XSolidBrush \TBstrut\\
+ Ring Init & 88.1 (0.0) & 59.4 (0.3) & 89.4 (0.1)  & 90.1 (0.1) & 94.4 (0.3)  & \XSolidBrush \TBstrut\\
\hline
S4D (our reproduction) & 91.5 (0.2) & 60.2 (0.3) & 86.4 (0.0) & 89.5 (0.0) & 94.2 (0.3) & 97.5 (0.0) \TBstrut\\
S5 (our reproduction) & 88.8 (0.1) & 58.5 (0.3) & 86.2 (0.1) & 88.9 (0.0) & 95.7 (0.1) & 96.0 (0.1) \TBstrut\\
\hline
S4 (paper results)& \it{91.1} &\it{59.6}  & \it{86.8} & \it{90.9} & \it{94.2} & \it{96.4} \TBstrut\\
S4D-LegS (paper results) & 89.9 & 60.5 & 86.2 & 89.5 & 93.1 & 91.9 \TBstrut\\
S5 (paper results) & \it{90.1} & \it{62.2}  & \it{89.3} & \it{91.4} & 95.3 & \it{98.6} \TBstrut\\
\hline
\end{tabular}
\end{sc}
\end{scriptsize}
\end{center}
\vskip -0.1in
\caption{\textit{Test accuracy of a linear diagonal complex RNNs under different parameterizations of the transition matrix~(see \S\ref{sec:diagonalization}). Performance directly improves the results in Tb.~\ref{tb:effect_nonlinearity}, and showcases the advantage of exponential~(polar) representation of $\Lambda$.  In bold font is the best parameterization option for linear RNN blocks. Ring Init denotes a changed initialization where $r_{\min}$ and $r_{\max}$ are tuned. Performance and Text and Retrieval task already aligns with S4 results in the dense setting~(c.f. Tb.\ref{tb:effect_nonlinearity} with Tb.~\ref{tb:effect_normalization}). No model with able to solve PathX, which requires normalization~(see Tb.\ref{tb:effect_normalization}).}}
\label{tb:effect_parametrization_exp_stddevs}
\end{table*}

\begin{table*}[ht]
\setlength{\tabcolsep}{5pt}
\begin{center}
\begin{scriptsize}
\begin{sc}
\begin{tabular}{|l||*{6}{c|}}
\hline
& \textbf{sCIFAR} & \textbf{ListOps} & \textbf{Text} &\textbf{Retrieval} & \textbf{Pathfinder} & \textbf{PathX} \TBstrut\\
\hline
Linear Dense RNN   & 72.2 (0.2) & 50.4 (0.2) & 89.1 (0.1)  & 89.1 (0.1)  & \XSolidBrush & \XSolidBrush \TBstrut\\
\hline
Diagonal Complex RNN  & 86.5 (0.1) & 58.8 (0.3) & 87.4 (0.3) & 87.8 (0.5) & \XSolidBrush & \XSolidBrush \TBstrut\\ \hline
Stable Exp Param w/ Ring Init & 88.1 (0.0)  & 59.4 (0.3) & 89.4 (0.1) & 90.1 (0.1) & 94.4 (0.3) & \XSolidBrush \TBstrut\\
$[r_{\min}, r_{\max}]$& [0.9, 0.99] & [0.0, 1.0] & [0.0, 0.9] & [0.5, 0.9] & [0.9, 0.999] &\TBstrut\\
\hline
$+\gamma$ Normalization \textbf{(LRU)} & 89.0 (0.1) & 60.2 (0.8) & 89.4 (0.1) & 89.9 (0.1) & 95.1 (0.1) & 94.2 (0.4) \TBstrut\\
$[r_{\min}, r_{\max}]$& [0.9, 0.999] & [0.0, 0.99] & [0.5, 0.9] & [0.5, 0.9] & [0.9, 0.999] & [0.999, 0.9999]\TBstrut\\
\hline
\hline
S4D (our reproduction) & 91.5 (0.2) & 60.2 (0.3) & 86.4 (0.0) & 89.5 (0.0) & 94.2 (0.3) & 97.5 (0.0) \TBstrut\\
S5 (our reproduction) & 88.8 (0.1) & 58.5 (0.3) & 86.2 (0.1) & 88.9 (0.0) & 95.7 (0.1) & 96.0 (0.1) \TBstrut\\
\hline
S4 (paper results) & \it{91.1} &\it{59.6}  & \it{86.8} & \it{90.9} & \it{94.2} & \it{96.4} \TBstrut\\
S4D-LegS (paper results) & 89.9 & 60.5 & 86.2 & 89.5 & 93.1 & 91.9 \TBstrut\\
S5 (paper results) & \it{90.1} & \it{62.2}  & \it{89.3} & \it{91.4} & 95.3 & \it{98.6} \TBstrut\\
\hline
\end{tabular}
\end{sc}
\end{scriptsize}
\end{center}
\vskip -0.1in
\caption{\textit{Effects of normalization on linear diagonal RNNs with stable exponential parameterization~(see \S\ref{sec:norm_pathx}). In bold is our best performing model, and we report the closely matching deep SSM results below. Tunings for our rings are also reported. Results showcase the advantage of taking initialization close to the unit circle under proper $\gamma$ normalization. For PathX, we initialize eigenvalues to have a phase range of $[0, \pi/10]$, for all other tasks we use a range of $[0, 2\pi]$ (see \S\ref{sec:norm_pathx}).}}
\label{tb:effect_normalization_stddevs}
\end{table*}

\section{Detailed experimental setup}
\label{app:detailed_experimental_setup}

In this section, we describe our experimental details.

\subsection{Architecture}

We consider the standard S4 architecture of \citet{gu2021efficiently} and replace the S4 layers with RNN layers or with S5 \citep{smith2022simplified} or S4D \citep{gu2022parameterization} layers for our baselines. We give an overview of the architecture used in Fig.\ref{fig:deep_rnn}. The input is first encoded into $H$ features, followed by a stack of residual blocks. For all our experiments, we use networks with a depth of 6 residual blocks. Each residual block consists of identity skip connection, and the residual path containing a normalization layer (in our case, we always use batch normalization in our experiments), followed by the RNN/SSM block. While using the ``post-norm'' option of adding the normalization layer after the skip and residual branches typically improves performance, we stick to this design due to this architecture being more scalable in general \citep{de2020batch}.

Each RNN/SSM block first contains the recurrent layer as described in Eqs.\eqref{eq:RNN} and \eqref{eq:S4-disc} in \S\ref{sec:preliminaries}. This is followed by a mixing layer. For all experiments except PathX, we use the GLU activation function \citep{dauphin2017language} with dropout as the mixing layer, similar to \citet{gu2021efficiently}. For PathX, we instead use a GLU activation function without one additional linear transform; the same as used by \citet{smith2022simplified} for their experiments.

We use bidirectional models for our experiments on PathFinder and PathX, using a similar setup as \citet{gu2021efficiently}, and use unidirectional models for the rest of our experiments.

\subsection{General experimental details}

We use AdamW as our optimizer \citep{loshchilov2017decoupled}. We use warmup for the learning rate, where we start from a value of $10^{-7}$ and increase the learning rate linearly up a specified value for the first 10\% of training. This is followed by cosine annealing for the rest of training down to a value of $10^{-7}$. 

We used a smaller learning rate for the RNN/SSM parameters $A$ and $B$. When using normalization in our RNNs, we also used a smaller learning rate on the normalization parameter $\gamma$. For our S5 and S4D baselines, we used a smaller learning rate for the discretization step size $\Delta$. This smaller learning rate was determined by multiplying the base learning rate by a factor $<1$ (See Tb.\ref{tb:hparams} for the learning rate factor used for each task).

We use weight decay for all parameters except the RNN/SSM parameters $A$ and $B$ (and $\gamma$ and $\Delta$ when applicable).

All experiments were carried out on accelerated hardware A100 GPUs.

\begin{table*}[ht]
\begin{center}
\begin{scriptsize}
\begin{sc}
\begin{tabular}{|l||*{8}{c|}}
\hline
Task & Depth & $H$& $N$ & Iterations & Batch size & LR factor & Weight Decay & Dropout \TBstrut\\ \hline
sCIFAR & 6 & 512 & 384 & 180k & 50 & 0.25 & 0.05 & 0.1 \TBstrut\\ \hline
ListOps & 6 & 128 & 256 & 80k & 32 & 0.5 & 0.05 & 0.0 \TBstrut\\ \hline
Text & 6 & 256 & 192 & 50k & 32 & 0.1 & 0.05 & 0.1 \TBstrut\\ \hline
Retrieval & 6 & 128 & 256 & 100k& 64 & 0.5 & 0.05 & 0.1 \TBstrut\\ \hline
PathFinder & 6 & 192 & 256 &  500k& 64 & 0.25 & 0.05 & 0.0 \TBstrut\\ \hline
PathX & 6 & 128 & 256 & 250k & 32 & 0.25 & 0.05 & 0.0 \TBstrut\\ \hline
\end{tabular}
\end{sc}
\end{scriptsize}
\end{center}
\vskip -0.1in
\caption{\textit{List of all the hyper-parameters used for each task for the LRU model.}}
\label{tb:hparams}
\end{table*}

\subsection{Hyperparameters}
\label{sec:hyperparams}
We closely followed the hyperparameter settings of the S5 model \citet{smith2022simplified} for all our experiments, with minimal additional tuning. For our S5 baseline, we tuned the model dimension $H$ and state dimension $N$, and used the optimal values for the LRU model as well. For the S4D baseline, we also tuned $H$ and $N$. For all our experiments, we tuned the base learning rate on a logarithmic grid of 2 to choose the optimal learning rate. We present the hyperparameters we used for each LRU experiment in Tb.\ref{tb:hparams}.

\subsection{Tasks}

We use the 6 tasks in the Long Range Arena benchmark for our experiments \citep{tay2020long}, with the only difference being we use colored sCIFAR images instead of the grayscale sCIFAR images used in LRA.

\section{Theoretical insights}
\label{app:theory}
We provide here theoretical groundings for some observations made in  \S\ref{sec:how_to}. We start by showing in \S\ref{app:expressivity} that, when interleaved with MLP blocks, stacked linear RNNs can model highly nonlinear dynamical systems. We provide two separate views that justify our findings: in \S\ref{app:spectral_view}, we provide a spectral explanation, while in \S\ref{app:koopman} we present a function-space prespective. Our results, combined with the observation that nonlinear RNNs are difficult to optimize~(\S\ref{app:opt}), provide a justification for the results in Tb.~\ref{tb:effect_nonlinearity}. Next, motivated by the results in Tb.~\ref{tb:effect_normalization} we in discuss in the same section optimization of linear RNN blocks, and show that exponential reparameterization can accelerate training.

\subsection{Expressivity of linear RNN stacks}
\label{app:expressivity}
In our sequence-to-sequence setting, it is a natural to seek models which~(at least in the width limit) are able to map inputs $u$ to outputs $y$~(last layer) using a flexible nonlinear transition map $T$ learned from data. Mathematically, a fully-expressive \textit{causal} model should be able to approximate $y_k = T(u_k, u_{k-1},\dots, u_1)$, where $T$ is an arbitrary nonlinear map. 

\subsubsection{Spectral perspective}
\label{app:spectral_view}
We show in this section how interleaving linear RNNs with MLPs in a deep architecture provides a flexible and modular recipe for the approximation of nonlinear transition maps.

\vspace{-3mm}
\paragraph{Spectral limitations of linear RNNs.} It is a standard result~\citep{li2022approximation} that \textit{linear} RNNs can approximate any shift-invariant \textit{linear} map $T$. In continuous-time, on the spectral domain, this property is easier to study: let $Y(\omega)$ and $U(\omega)$ be the Fourier transforms for two continuous-time signals $u,y:\R\to\R$. If there exists a function $H:\R\to\R$ such that $Y(\omega) = H(\omega) U(\omega)$, then this can be approximated by a continuous-time linear RNN $\dot x = A x + B u$ for some coefficients $A\in\R^{N\times N}, B\in\R^{N\times 1}$, and the approximation can be made arbitrarily accurate as $N\to\infty$. However, one thing a linear RNN \textit{cannot do} is store information under frequencies which are not present in the input signal: if the input is a sine wave of a certain frequency, the output will be a scaled and shifted sine wave of the \textit{same frequency}. 

\vspace{-2mm}
\paragraph{Spectral effects of interleaving with MLPs.} In our architecture~(Fig.\ref{fig:deep_rnn}) an activation function, as well as a linear position-wise layer, is placed right after each RNN output. As can be seen in Fig.~\ref{fig:fft_after_sin}, this operation causes spectral leakage: information gets copied over different frequency components.

\begin{figure}[ht]
    \centering
    \includegraphics[height=0.27\textwidth]{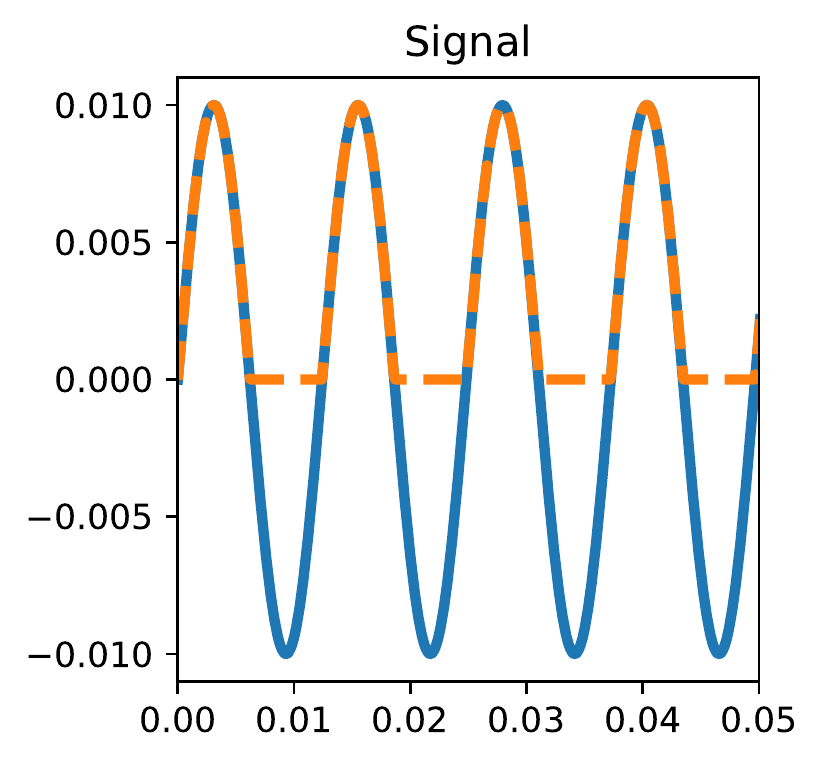}
    \includegraphics[height=0.27\textwidth]{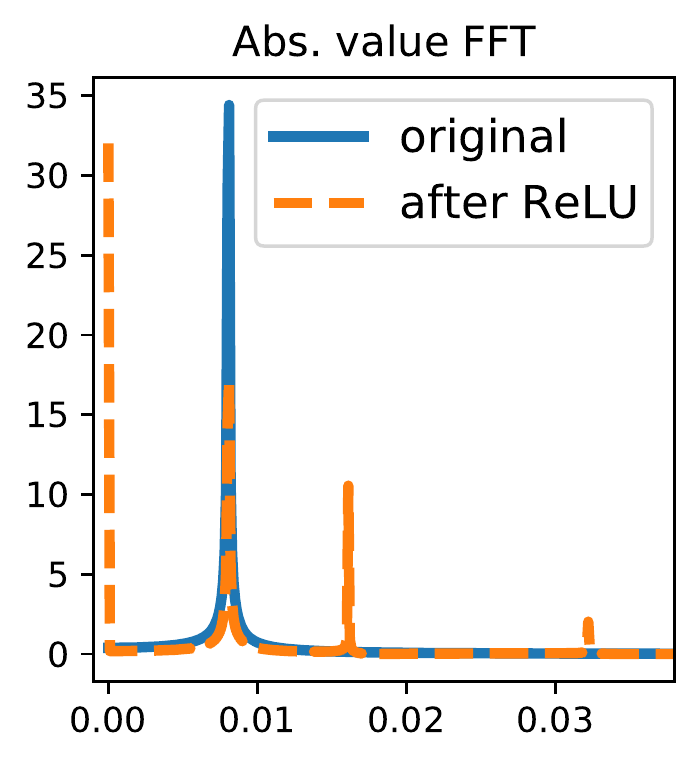}
    \caption{\textit{ReLU nonlinearity leaks information from the original signal to higher frequencies, as shown formally in Prop.~\ref{prop:relu_effect}.}}
    \label{fig:fft_after_sin}
\end{figure}
The behavior shown in Fig.~\ref{fig:fft_after_sin} can be characterized exactly:
\begin{restatable}[Spectral effect of ReLU]{prop}{spectra_relu_conv} Let $u:\R\to\R$ be a continuous-time signal. Let $P_i$ be the $i$-th region activated by the ReLU applied to $u$, and let us write $P_i = [p_i-L_i, p_i+L_i]$. Then
 \vspace{-2mm}
\begin{equation}
    \mathcal{F}_{\text{ReLU}(u)} = \mathcal{F}_{u}(\omega) \star\left[\sum_{i} 2L_i e^{-i \omega p_i} \text{sinc}(\omega L_i)\right].
\end{equation}
where $\mathcal{F}$ denotes the Fourier transform, $\star$ the convolution operation and $\text{sinc}(x):= \sin(x)/x$.
\label{prop:relu_effect}
 \end{restatable}
 This result is simple to parse: the Fourier transform of a ReLU activated signal is equal to the Fourier transform before the ReLU, convolved with a kernel which transports information to higher frequencies --- an operation which is \textit{impossible} for linear RNNs, even as the width increases. As such, \textbf{introducing an MLP completes the list of requirements for approximations of a nonlinear transition map: frequencies can be scaled up and down arbitrarily by the RNN, and can then be translated in the space using the ReLU}. As depth increases, these operations can be combined in a modular fashion, leading to highly nonlinear dynamics using easy-to-learn linear blocks, interleaved with simple activations. 
 
 To conclude, we provide a proof for the proposition above.
 
 \begin{proof}
 Recall that multiplications in the time domain are convolutions in the frequency domain.
\begin{align}
    u_1(t) \cdot u_2(t) &= \mathcal{F}^{-1}_{U_1}(t)\cdot\mathcal{F}^{-1}_{U_2}(t)\\
    &=  \left(\int_{-\infty}^{\infty} U_1(\nu)e^{i \nu t}d\nu\right)\cdot\left(\int_{-\infty}^{\infty} U_2(\xi)e^{i\xi t}d\xi\right)\\
    &=  \int_{-\infty}^{\infty} U_1(\nu)\left(\int_{-\infty}^{\infty} U_2(\xi)e^{i (\xi+\nu) t}d\xi\right)d\nu\\
    &=  \int_{-\infty}^{\infty} U_1(\nu)\left(\int_{-\infty}^{\infty} U_2(\omega-\nu)e^{i \omega t}d \omega\right)d\nu\\
    &=  \int_{-\infty}^{\infty} \left(\int_{-\infty}^{\infty} U_1(\nu) U_2(\omega-\nu)d\nu\right)e^{i \omega t}d \omega\\
    &= \mathcal{F}^{-1}_{U_1\star U_2}(t).
\end{align}

Let now $u_1=u$ and $u_2 = \chi(u_1>0)$, then $u_1\cdot u_2 = \text{ReLU}(u)$. Next, let $P_i$ be the $i$-th region activated by the ReLU, and let us write $P_i = [p_i-L_i, p_i+L_i]$. We can write $\chi(u_1>0) = \sum_{i} \chi_{[p_i-L_i, p_i+L_i]}$.

Recall now the following basic properties:
\begin{enumerate}
    \item $\mathcal{F}_{x(t-t_0)}(\omega) = e^{-i\omega t_0} \mathcal{F}_{x(t)}(\omega)$.
    \item The Fourier transform of a rectangular pulse between $-\tau$ and $\tau$ is $2\tau\cdot \text{sinc}(\omega\tau)$, where $\text{sinc}(x) = \sin(x)/x$.
\end{enumerate}
Therefore, we have
\begin{equation}
    \mathcal{F}_{\chi_{[p_i-L_i, p_i+L_i]}}(\omega) = e^{-i \omega p_i} \mathcal{F}_{\chi_{[-L_i,L_i]}}(\omega) = 2L_i e^{-i \omega p_i} \text{sinc}(\omega L_i).
\end{equation}
This concludes the proof:
\begin{equation}
    \mathcal{F}_{\text{ReLU}(u)} = U \star\left[\sum_{i} 2L_i e^{-i \omega p_i} \text{sinc}(\omega L_i)\right].
\end{equation}
 \end{proof}

\subsubsection{Insights from Koopman operator theory}
\label{app:koopman}
 We show how Koopman operator theory~\citep{koopman1932dynamical}, combined with recent advances in dynamic mode decomposition~\citep{schmid2010dynamic,kutz2016dynamic,williams2015data}, can provide a solid theoretical foundation for understanding the class of functions that can be approximated by linear RNNs, interleaved with MLPs. Our notation and results are based on~\citet{korda2018convergence,mauroy2020koopman}.

\paragraph{Basic theory.} Consider a discrete-time nonlinear dynamical system $x_{k+1} = S(x_k)$, where $S:\mathbb{R}^n\to\mathbb{R}^n$ is a sufficiently regular map. The Koopman operator $\mathcal{K}_S$ for the dynamical system $S$ prescribes the evolution of any observable (measurement) $f:\mathbb{R}^n\to \mathbb{C}$:
\begin{equation}
    (\mathcal{K}_S f)(x) := f(S(x)).
\end{equation}
For instance, let us consider $n=1$ and the observable $f(x) = \sin(x)$: the Koopman operator is the map that takes $\sin(\cdot) \overset{\mathcal{K}_S}{\mapsto} \sin(S(\cdot))$, i.e. \textit{advances the measurement} $f$ one step forward in time.\\ 
The crucial property of the Koopman operator is that it is\textbf{ linear} and bounded~\citep{mauroy2020koopman}: let $f_1,f_2$ be two observables, then
\begin{align}
    \mathcal{K}_S(\alpha f_1 + \beta f_2)(x) &= (\alpha f_1 + \beta f_2)(S(x))\\
    &= \alpha f_1(S(x)) + \beta f_2(S(x))\\ &= \alpha (\mathcal{K}_S f_1)(x) + \beta (\mathcal{K}_S f_2)(x). 
\end{align}
If $S$ is regular enough, i.e. if the Hilbert space of observables can be chosen such that $\mathcal{K}$ only has point spectrum, then the spectral theory of bounded linear operators in Hilbert spaces implies that $\mathcal{K}_S$ is diagonalizable --- i.e. any observable $f$ can be expanded in terms of eigenfunctions of $\mathcal{K}_S$, where the Koopman acts linearly. We recall the definition: $\phi_\lambda:\mathbb{C}^n\to\mathbb{C}$ is an eigenfunction of $\mathcal{K}_S$ with eigenvalue $\lambda\in\mathbb{C}$ if $\mathcal{K}_S\phi_\lambda = \lambda \phi_\lambda$ --- i.e if the system measured on $\phi$ evolves linearly. Since the eigenfunctions of $\mathcal{K}_S$ form a basis for $L_2$, for any observable $f:\mathbb{C}^n\to\mathbb{C}$, there exist complex numbers ${\nu_1,\nu_2,\cdots}$ such that one can write~\citep{mauroy2016global}
\begin{equation}
    \mathcal{K}_S f(x) = \mathcal{K}_S\left(\sum_{j=1}^\infty \nu_j \phi_j \right)(x) =\sum_{j=1}^\infty \lambda_k \nu_j \phi_j(x). 
\end{equation}
Since also the identity measurement map $x\mapsto x$ can be decomposed into eigenfunctions of $\mathcal{K}_S$ coordinate-wise, we have the following: assuming $x_{k+1} = S(x_k)$, with $x\in\mathbb{R}^n$, for any $k\in\mathbb{N}$ we have
\begin{equation}
    x_k = V \Lambda^k \Phi(x_0),
\end{equation}
where, with slight abuse of notation, $\Phi:\R^n\to\mathbb{C}^\infty$ is a vector of functions with the $j$ coordinate defined as $(\Phi)_j := x \mapsto \phi_j(x)$, and $V\in\mathbb{C}^{n\times\infty}$~(often named the Koopman modes matrix) is the infinite dimensional matrix such that, for the observable $f_i:x\mapsto x_i$, one has $f_i(x) = \sum_{j=1}^\infty V_{ij}\phi_j(x)$.

\paragraph{Basic Theory Summary.} In essence, Koopman operator theory, provides the following guarantee: \textbf{\textit{any sufficiently regular nonlinear autonomous dynamical system can be made linear under a high-dimensional nonlinear blow-up of the state-space. Sounds familiar? This is exactly what a wide MLP + Linear RNN can do}}. Moreover, to take the system back to the original coordinate system, one just needs a linear projection with matrix $V$. In practice, for identification and diagnosis of nonlinear systems~(e.g. in machanical engineering), this approach is used in a truncated version, where the finite class of dominant eigenfunctions is constructed by using the dynamic mode decomposition (DMD) algorithm from Hermite Polynomials~\citep{schmid2010dynamic,kaiser2021data}.

\paragraph{Extension to nonlinear systems with inputs.} Several options exist for extending Koopman operator theory to systems with inputs~\citep{surana2016koopman,proctor2018generalizing,kaiser2021data,korda2020koopman}. Here, we briefly outline the approach of~\citep{korda2020koopman}. Let $S:\R^n\times \R^m\to\R^n$ be a nonlinear function which evolves the state of the system as $x_{k+1}=S(x_k,u_k)$, where $(u_k)_{k=1}^{\infty}\in\ell_2(\R^m)$ is the input sequence. We wish to take this nonlinear dynamical system with inputs to linear form in the infinite-dimensional space of observables $f$ of the form $\R^n\times\ell_2(\R^m)\to\mathbb{C}$. Let $\mathcal{L}$ denote the left shift operator $\tilde u = (u_0, u_1,\dots)\mapsto \mathcal{L}(\tilde u) = (u_1, u_2,\dots)$, then one can define the Koopman operator for any observable $f$ as follows:
\begin{equation}
    \mathcal{K}_S f(x, \tilde u) = f(S(x, u_0), \mathcal{L}(\tilde u)).
\end{equation}
This operator is again linear and bounded for regular enough $S$~\citep{korda2020koopman} --- hence the analysis in the autonomous setting carries out also in this case. In particular, using the notation in the last paragraph:
\begin{equation}
    x_k = V \Lambda_{(x,u)}^k \Phi(x_0, \tilde u),
\end{equation}
where $\Lambda_{(x,u)}$ is a diagonal complex infinite-dimensional matrix which contains the eigenvalues corresponding to the eigenfunctions of the extended state $\Phi(x_0, \tilde u)$.

\paragraph{Implication for deep RNNs.} In essence, Koopman operator theory, provides the following guarantee:\textit{ any regular nonlinear dynamical system is representable by a linear RNN after proper nonlinear reparameterization of the inputs} --- which can be performed by an MLP. While we believe this connection is conceptually solid and gives substantial insights into our architecture, a quantitative discussion would require substantial technical efforts perhaps linked to recent contributions from the statistical learning community~\citep{kostic2022learning}.

 \subsection{Optimization of recurrent blocks}
\label{app:opt}

In this subsection we back-up some of our claims about optimization of linear RNNs with experimental findings on toy examples. Our purpose is to confirm validity of our intuition outside the deep learning setting, without architecture-dependent confounders: i.e on vanilla RNNs with one layer.

\vspace{-3mm}
 \paragraph{Recurrent nonlinearities slow down gradient descent. }In \S\ref{sec:how_to} and \S\ref{app:expressivity} we showed how linear RNNs can be used as elementary recurrent blocks for the purpose of modeling complex nonlinear dynamics when stacked in deep architectures. Similarly, the results in~\citep{li2022makes} indicate that, to achieve S4 performance, one can equivalently replace the recurrent core with a collection of convolutions parametrized by filters. While a single-layer level, a (dense) RNNs~(Eq.\ref{eq:RNN}) with $\tanh$ or sigmoid activation can express convolutions with filters~\citep{wangeffects}, the results in Tb.~\ref{tb:effect_nonlinearity}~(and Fig.~1(a) in~\citet{wangeffects}) indicate an advantage on test accuracy from dropping such nonlinearities in the recurrence --- i.e. of making the RNN linear. Motivated by this, in Fig.~\ref{fig:tanh_vs_lin_kernel_learning} we consider the problem of learning a single one-dimensional convolution kernel with a single layer RNN, and compare performance of linear and $\tanh$ activations. The sequence length in this problem was $100$, and our data consists in 32 input-output one-dimensional trajectories, where the output is the result of a convolution with the kernel of elements $h_k:= \frac{1}{10}\exp(-0.015\cdot k)\cos(0.04\cdot k)^2$, which induces moderate-length dependencies in the data~(see bump in the kernel in Figure~\ref{fig:tanh_vs_lin_kernel_learning} at $k=70$). The 32 input sequences are generated sampling random $a,c$ parameters on a range and have form $\sin(0.05\cdot a\cdot k)\cos(0.05\cdot c\cdot k)^2$. Outputs are generated by convolving each input by $h$. Learning is performed using the Adam optimizer~\citep{kingma2014adam} with standard momentum parameters.
 
 Interestingly, already on this simple task, linear RNNs outperforms the $\tanh$ variant even after careful tuning of the stepsize. While the input-output map the system had to approximate is linear~(i.e. a convolution), this result still indicates that on deep architectures, where the MLPs interleaving RNNs can quickly perform position-wise nonlinearities lifting the function approximation class~(see \S\ref{app:expressivity}), linear RNNs are preferrable.
 
 \begin{figure}[ht]
 \vspace{-2mm}
     \centering
     \hspace{1mm}
      \includegraphics[height=0.26\textwidth]{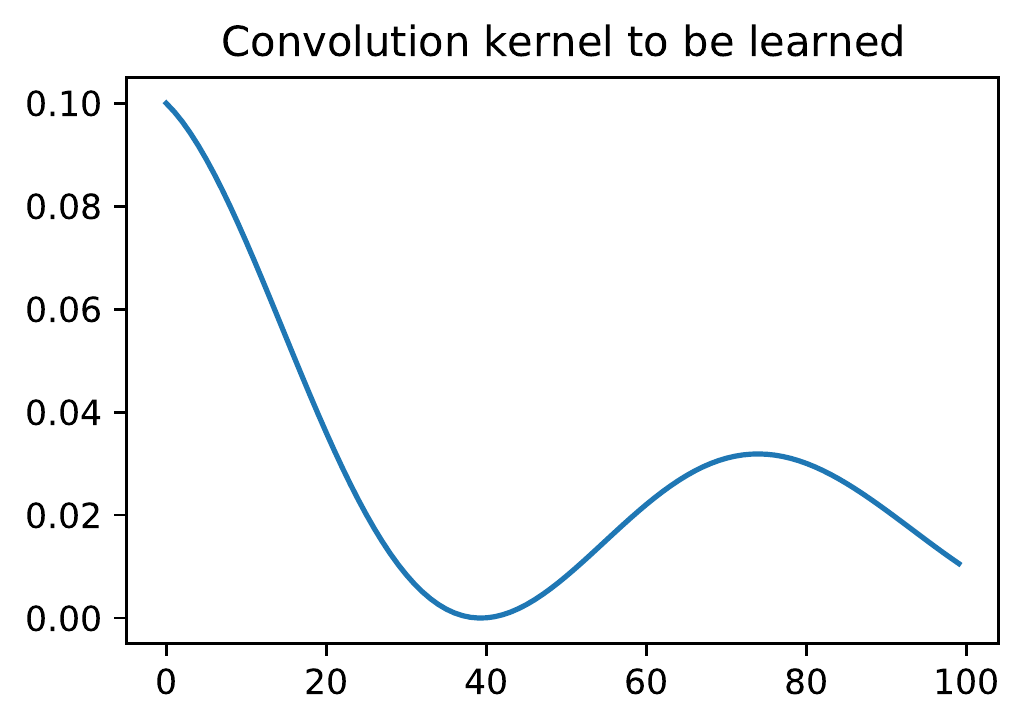}
     \includegraphics[height=0.26\textwidth]{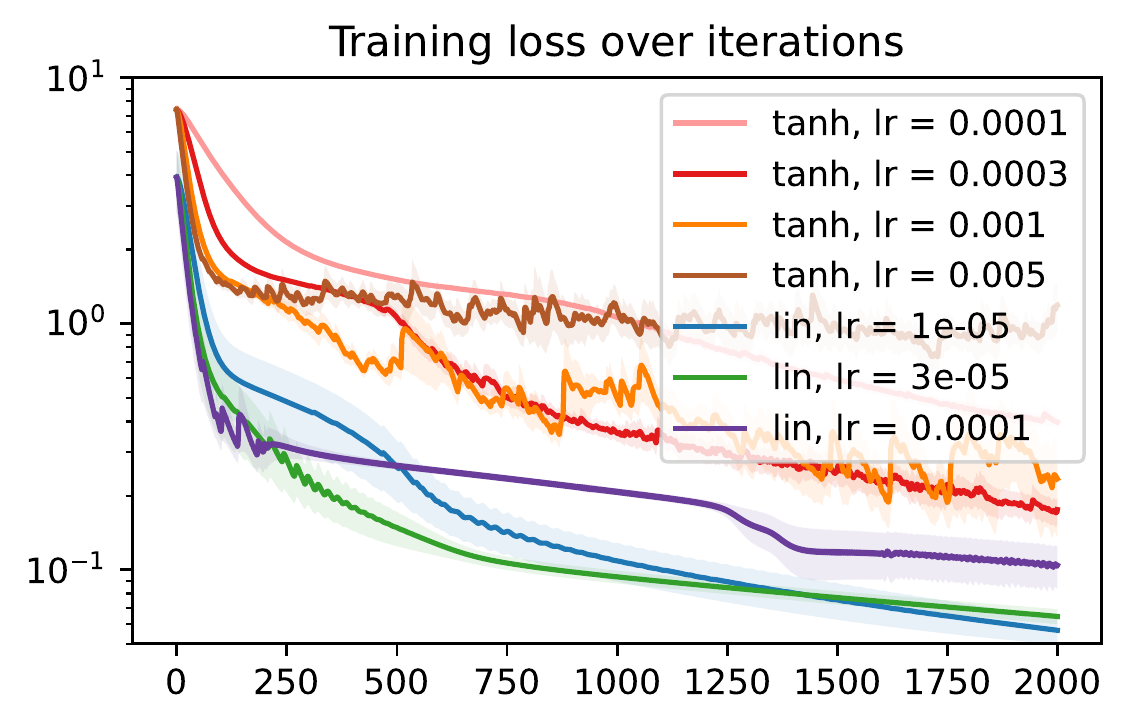}\\
     \vspace{-2mm}
     \caption{\textit{Learning with Adam a one-dimensional convolution with a length-$100$ kernel using a single-layer RNNs with linear or $\tanh$ recurrent activations and 100-dimensional hidden state. Initialization is performed using Glorot on all quantities for both options. For all learning rates in our grid, the linear variant is faster to converge.}}
     \label{fig:tanh_vs_lin_kernel_learning}
     \vspace{-3mm}
 \end{figure}
 
 \vspace{-3mm}
 \paragraph{Benefits of exponential parameterization.} Our experimental results in \S\ref{sec:exponential} indicete that linear RNN cores can be more effectively learned under exponential parameterization of the eiganvalues: $\lambda = \exp(-\nu + i\theta)$. To understand the reason behind this phenomenon, we go back at the classical (hard) problem of learning powers~\citep{bengio1994learning}, crucially linked with linear RNN models~(see Eq.~\eqref{eq:lin_rnn_unroll}). For a specific planted solution $\lambda^* = \lambda_r^* + i \lambda^*_i = \exp(-\nu^*+i\theta^*)$, we consider the problem of minimizing the loss $L(\hat\lambda) = \frac{1}{2}|\hat\lambda^k -(\lambda^*)^k|^2$, where $k=100$ and $\hat\lambda$ is generated from two real parameters following standard~( real + imaginary) or exponential parameterization. Note that in this paragraph $\lambda^*\in \Cmp$ denotes the solution, not the complex conjugate of $\lambda$. In Fig.~\ref{fig:learning_powers}, we show that as the target phase $\theta^*$ approaches $\pi/2$ (i.e. $\lambda^*$ gets close to the imaginary axis), standard parameterization slows down learning, as the corresponding landscape gets non-axis-aligned --- a feature that does not match well the inner workings of the Adam optimizer\footnote{For this problem, vanilla gradient descent cannot be effectively used as the landscape is highly non-convex, with challenging curvature vanishing as $|\lambda|$ approaces $0$. }, which is a diagonal preconditioner~\citep{kingma2014adam}. Instead, under exponential parameterization, the effects of phase and magnitude parameters on the powers of $\lambda$ are more efficiently decouped: for example, while the real part of $\lambda^k$ is simply $\exp(-k\nu)$ using exponential parameterization, if standard parameterization is used, $\text{Re}\left[\lambda^k\right]$ is a function of both $\lambda_r$ and $\lambda_i$. We noticed that the performance difference gets most pronounced when the system has to learn how to ``turn'': i.e. the initialization magnitude is correct, but the position on the complex plane is not~(this is the precise setting for Figure~\ref{fig:learning_powers}): while for standard parameterization changing the phase $\theta^*$ requires a careful balance between real and imaginary components, for exponential parameterization gradients are fully aligned with the phase parameter. This makes the learning more flexible, a feature which we observed necessary in our experiments on the Long Range Arena, see \S\ref{sec:exponential} and Tb.\ref{tb:effect_parametrization_exp}.

  \begin{figure*}
     \centering
     \includegraphics[width=0.24\textwidth]{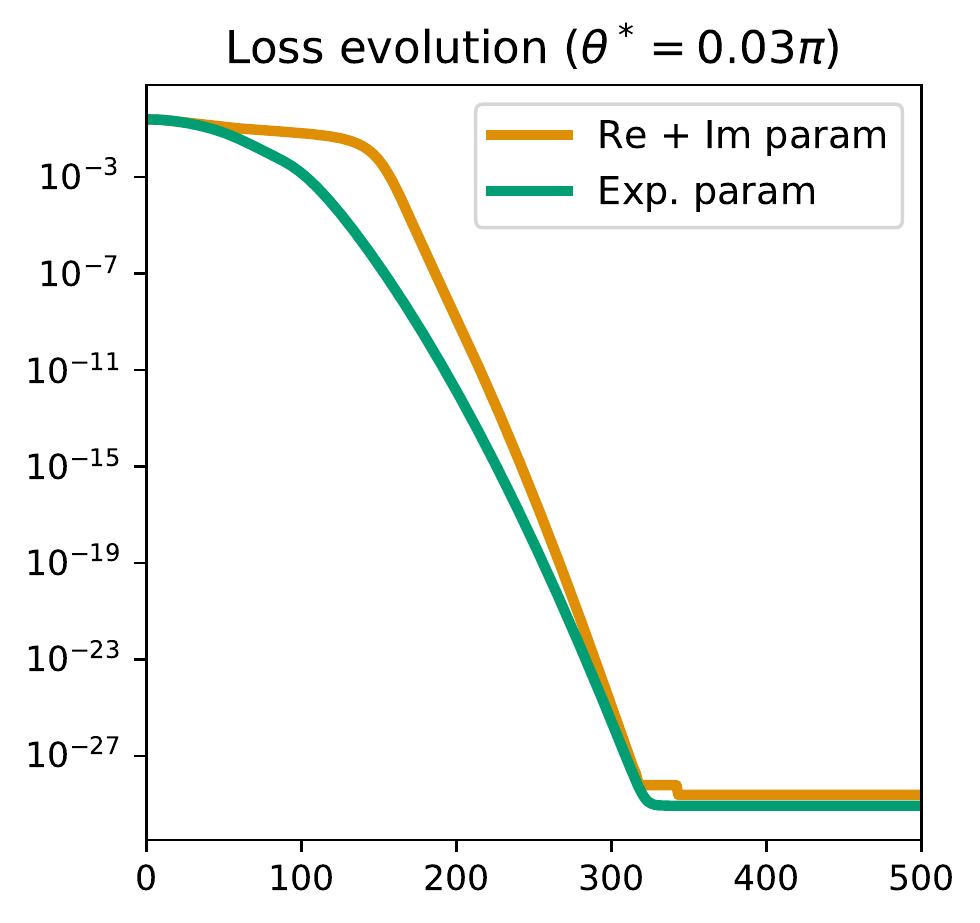}
     \includegraphics[width=0.24\textwidth]{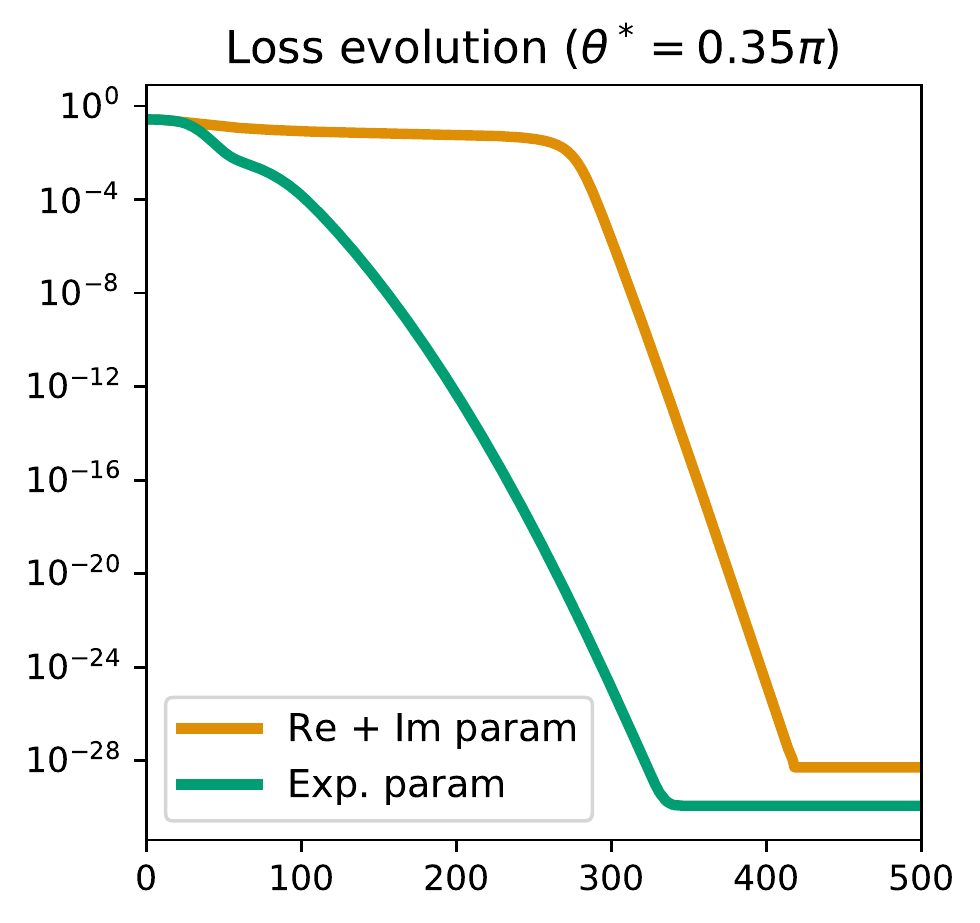}
     \includegraphics[width=0.25\textwidth]{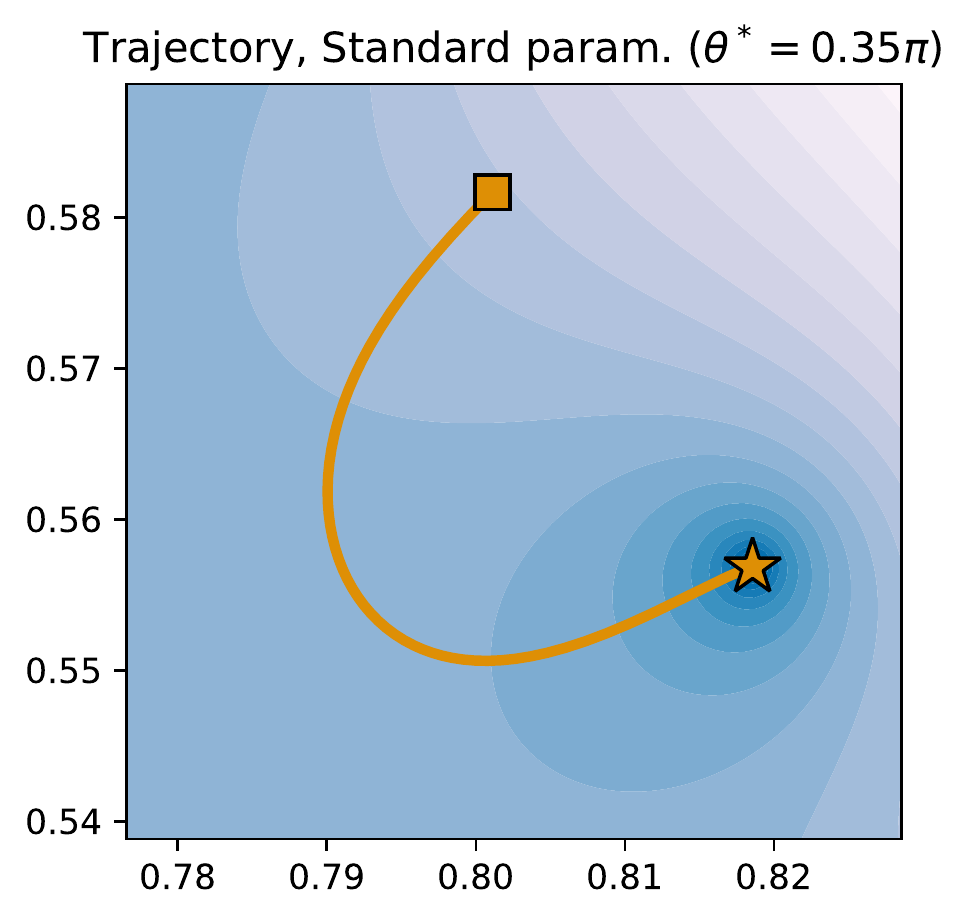}
     \includegraphics[width=0.25\textwidth]{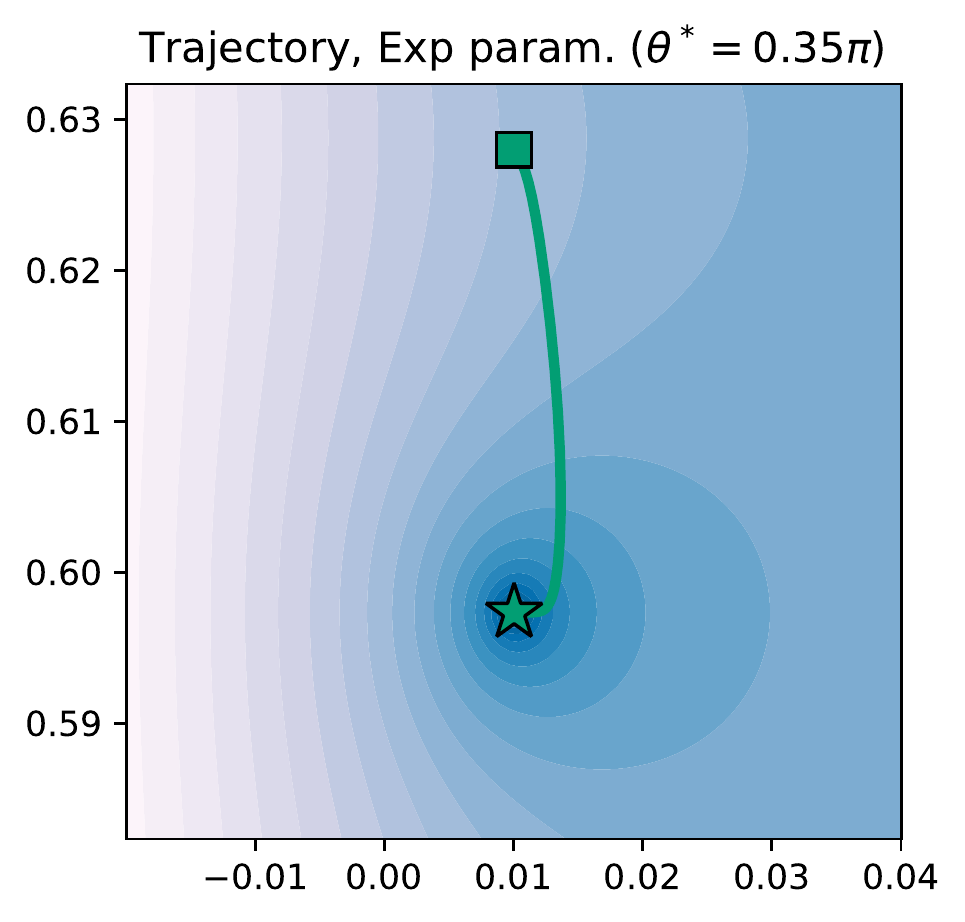}
     \caption{\textit{Exponential parametrization helps when learning a single complex eigenvalue $\lambda^* = \exp(-\nu^* + i\theta^*)$, exponentiated $100$ times. As $\lambda^*$ gets close to the purely imaginary setting $\theta^*=\pi/2$, the geometry of the loss landscape under standard real+imaginary parametrization becomes suboptimal for the Adam optimizer, which works best in the axis-aligned setting~(exponential parametrization). In the plot, the square denotes initialization , while the star denotes the solution after 500 iterations.}}
     \label{fig:learning_powers}
 \end{figure*}
 
\subsection{On alternatives to using complex numbers}
\label{app:real_jordan_form}
In this subsection, we show how to derive the canonical real form for a non-symmetric real-valued matrix $A$, which we assume to be diagonalizable in the complex domain~(always true up to arbitrary small perturbation of the entries~\citep{axler1997linear}). This derivation is classical and can be found in many textbooks under the context of \textit{real Jordan form}~(more general), see e.g.~\citet{weintraub2009jordan}. Here, we present a simplified discussion.

After diagonalizing $A$, we retrieve a set of purely real eigenvalues~(each with multiplicity 1 up to vanishing perturbations) with corresponding \textit{real} eigenvectors, and pairs of complex conjugate eigenvalues, with corresponding \textit{complex conjugate} eigenvectors.\\
We recall a proof for the facts above: let $^*$ denote the elementwise complex conjugate of any complex quantity. This operation clearly commutes with multiplication. If $\lambda\in\Cmp$ is an eigenvalue of $A\in\R^{N\times N}$ with eigenvector $v\in\Cmp^N$, then since $A$ is real-valued we have $A v^* = (A^* v)^* =(A v)^* = (\lambda v)^* = \lambda^* v^*$. Hence, $\lambda^*$ is an eigenvalue with eigenvector $v^*$. This also shows that there always does exist a real eigenvector corresponding to each real eigenvalue: let $v\in\Cmp^N$ be a complex eivengvector with real eigenvalue $\lambda$, then $v+v^*\in \R^N$ is an eigenvector with eigenvalue $\lambda$ since, again using the fact that $A$ is real,  $A(v+v^*) = Av + Av^* = Av + (Av)^* = \lambda(v+v^*)$ .

The action of $A$ on its real eigenvectors~(with real eigenvalues) is trivial and analogous to the symmetric case --- this corresponds to a diagonal entry in the diagonalized version of $A$. For the subspaces spanned by complex eigenvalues, the discussion is more interesting: let $\lambda, \lambda^*$ be a pair of conjugate eigenvalues with corresponding eigenvectors $v, v^*$. Collect $v, v^*$ in a $N\times 2$ matrix $V$, then
\begin{equation}
    A V = V\begin{pmatrix}\lambda& 0\\0&\lambda^*\end{pmatrix}  =: V \Lambda 
\end{equation}
Let us now choose a different \textit{real} basis for the columns of $V$, the real and imaginary parts of $v$: $\tilde V =[\text{Re}(v),\text{Im}(v)]$. Note that this is a basis, since $v, v^*$ are linearly independent and can be both written as (complex-weighted) linear combination of real and imaginary parts of $v$. Now note that 
\begin{align*}
    A \cdot\text{Re}(v) &= \frac{1}{2}A(v+v^*)\\
    &= \frac{1}{2} \left(\lambda v +\lambda^*v^*\right)\\
    &= \text{Re}(\lambda v) \\
    &= \text{Re}\left[(\text{Re}(\lambda) + i\text{Im}(\lambda))(\text{Re}(v) + i\text{Im}(v))\right]\\
    & = \text{Re}(\lambda)\text{Re}(v) - \text{Im}(\lambda)\text{Im}(v).
\end{align*}

Similarly,
\begin{align*}
    A \cdot\text{Im}(v) &= \frac{1}{2}A(v-v^*)\\
    &= \frac{1}{2} \left(\lambda v -\lambda^*v^*\right)\\
    &= \text{Im}(\lambda v) \\
    &= \text{Im}\left[(\text{Re}(\lambda) + i\text{Im}(\lambda))(\text{Re}(v) + i\text{Im}(v))\right]\\
    & = \text{Re}(\lambda)\text{Im}(v) + \text{Im}(\lambda)\text{Re}(v).
\end{align*}

This shows that the action of $A$ on the new \textit{real} basis $\tilde V$ is of simple form:

\begin{equation}
    A \tilde V = \tilde V\begin{pmatrix}\text{Re}(\lambda)& -\text{Im}(\lambda)\\\text{Im}(\lambda)&\text{Re}(\lambda)\end{pmatrix}  =: \tilde V \tilde \Lambda 
    \label{eq:equiv_complex}
\end{equation}

This discussion shows that there exist a simple invertible change of basis~(from $V$ to $\tilde V$ for all pairs of conjugate eigenvalues) which makes takes the system back to a simple decomposition in the real domain, both in terms of eigenvalues and eigenvectors --- one simply has to replace all diagonal blocks of form $\Lambda$ with $2\times 2$ matrices $\tilde \Lambda$.

The careful reader might recognize that, in the resulting system, matrix multiplication for the $2\times 2$ blocks is algebraically equivalent to multiplication of the corresponding complex numbers. Hence, while complex numbers are not \textit{per-se} needed to find a simple representation of non-symmetric matrices, they are convenient to work with since the matrix in Eq.~\eqref{eq:equiv_complex} is structured: has 4 entries but can be represented using just two --- real and imaginary parts, exactly what a complex number stores in memory.

\newpage
\section{Proofs}
\label{app:proofs}
In this section we provide proofs for the propositions listed in the main paper.

\subsection{Proof of Lemma~\ref{lemma:sampling_exp}}
We provide here a proof for the following sampling lemma.
\sampling*
\begin{proof}
First, note that one can sample phase and magnitude independently by symmetry of the target distribution. Phase sampling can trivially performed through scaling a uniform distribution. 

Next, we consider sampling the magnitude. The area of the ring in between $r_{\min}$ and $r_{\max}$ is $\pi(r_{\max}^2-r_{\min}^2)$, while the cumulative distribution function for the radius distribution is such that $F_r(r_{\min}) = 0$, $F_r(r_{\max}) = 1$ and for $r\in[r_{\min},r_{\max}]$ we therefore have
\begin{equation}
    F(r) = \frac{r^2-r_{\min}^2}{r_{\max}^2-r_{\min}^2}.
\end{equation}
Under parametrization of $r$ using the exponential, $r = e^{-\nu}$, one gets
\begin{equation}
    F(r) = \frac{e^{-2\nu}-r_{\min}^2}{r_{\max}^2-r_{\min}^2}.
\end{equation}
Finally, we use the inverse sampling theorem~(see e.g. \citet{vogel2002computational}): one can sample $\nu$ using the formula $\nu = F^{-1}(u)$, where $u$ is uniform on $[0,1]$. By setting
\begin{equation}
    u = \frac{e^{-2\nu}-r_{\min}^2}{r_{\max}^2-r_{\min}^2},
\end{equation}
we get
\begin{equation}
    e^{-2\nu} = (r_{\max}^2-r_{\min}^2) u + r_{\min}^2,
\end{equation}
from which follows that $\nu = -\frac{1}{2}\log((r_{\max}^2-r_{\min}^2) u + r_{\min}^2)$.
\end{proof}

\subsection{Proof of Proposition~\ref{prop:r_min_max_effect}}
\label{app:normalization}

Validity of this proposition is verified numerically in Figure~\ref{fig:verification_gain_prop}.

\begin{figure}[ht]
    \centering
    \includegraphics[width=0.4\textwidth]{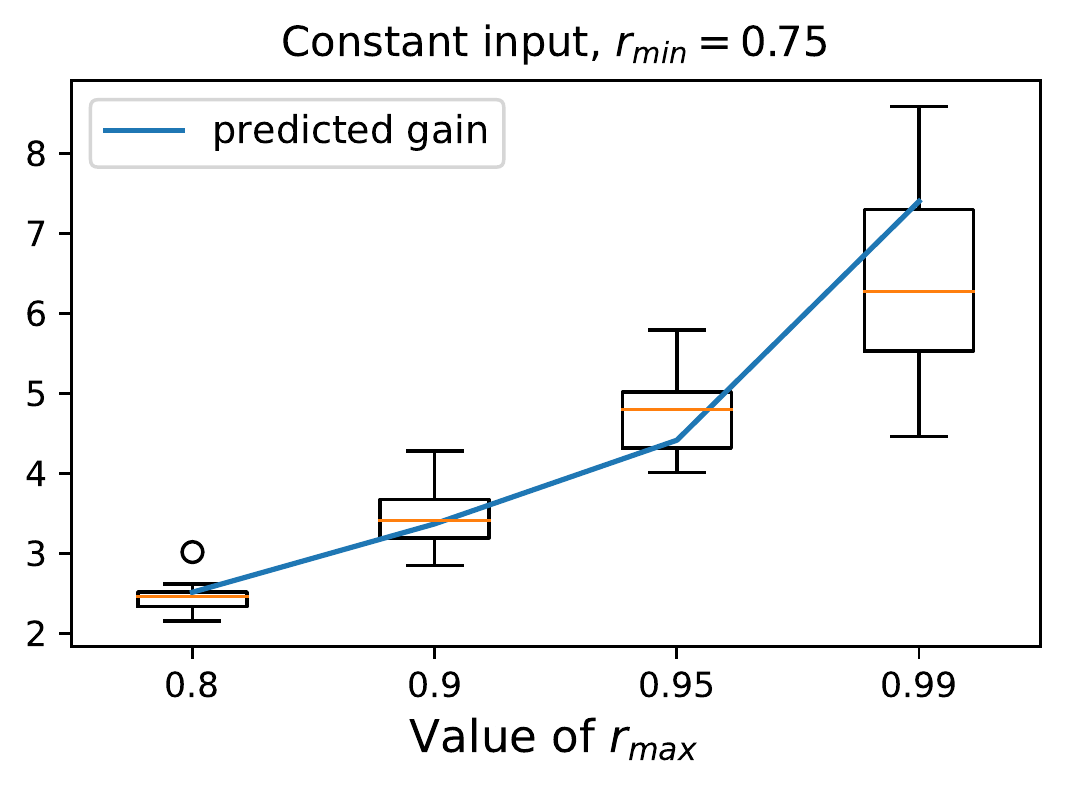}
    \includegraphics[width=0.4\textwidth]{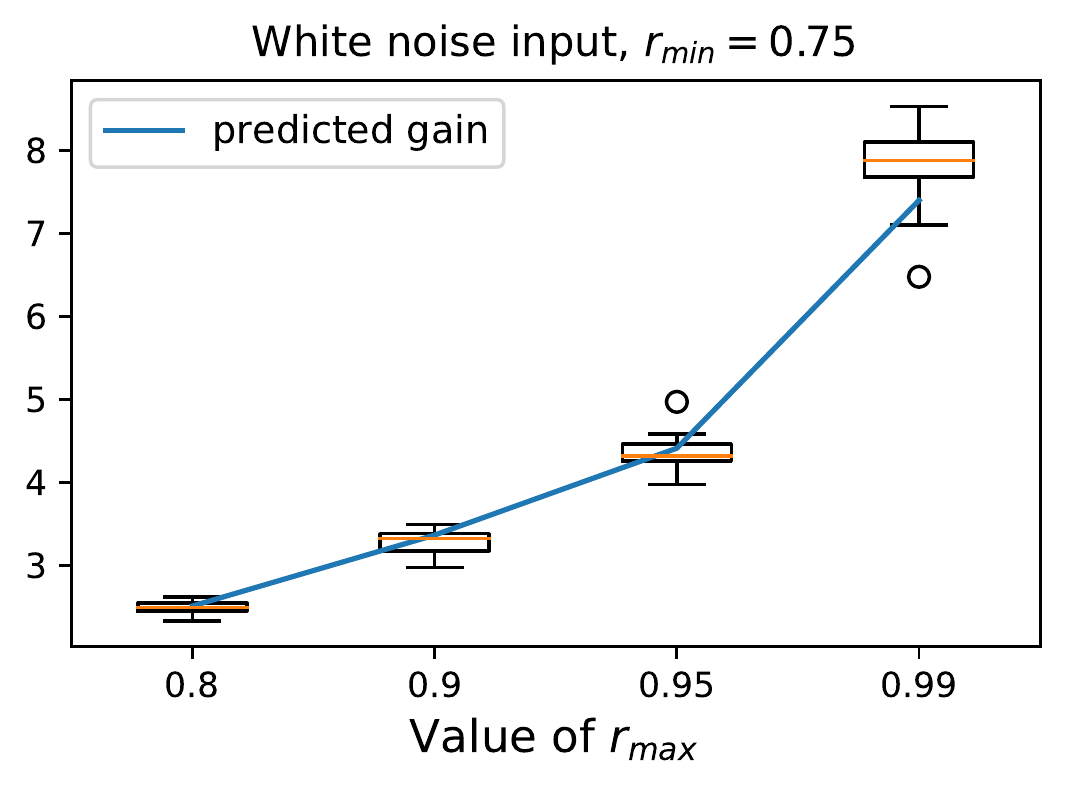}
    \caption{\textit{Numerical simulation for gain formula derived in Proposition~\ref{prop:r_min_max_effect}. Here we chose $N=500$, $L=10k$~(sequence length) and plotted statistics for 10 runs with boxplot indicating median and (5,95) percentile. Indicated in blue line is our prediction. The formula holds both for constant and random input, yet we notice that it is more accurate in the random input setting.}}
    \label{fig:verification_gain_prop}
\end{figure}

\blowup*

\begin{proof}
Assume first (most difficult case) that $u_k $ is constant, i.e. such that $B u_k =: \tilde u$ for all $k$. Then,
\begin{align}
    \|x_\infty\|_2^2 &= \sum_{n=1}^{\infty}\sum_{m=1}^{\infty} \tilde u_{k-m}^*\left(\Lambda^m\right)^* \Lambda^n \tilde u_{k-n}\\ &= \tilde u^*\left[\sum_{n=1}^{\infty}\sum_{m=1}^{\infty} \left(\Lambda^m\right)^* \Lambda^n\right] \tilde u.
\end{align}

Note that $\Lambda = \text{diag}(\lambda_1,\dots,\lambda_N)$ is diagonal with equally distributed entries on the disk between radii $r_{\min}$ and $r_{\max}$. One can then sample a generic entry $\lambda$ using the change of variables formula for probabilities~\citep{jeffreys1998theory} as follows~(see also Lemma~\ref{lemma:sampling_exp}):
\begin{equation}
    \lambda = r^{\frac{1}{2}}e^{i2\pi\theta}, \quad r\sim\mathcal{U}[r_{\min}^2,r_{\max}^2],\quad \theta\sim\mathcal{U}[0,1],
\end{equation}
Where crucially $r$ and $\theta$ are independent. Let $\mathbb{T}(r_{\min}, r_{\max}) = \{\lambda\in\Cmp: |\lambda|\in[r_{\min},r_{\max}]\}$. We need to study the following quantity:
\begin{align}
    \E_{\lambda\sim\mathbb{T}(r_{\min}, r_{\max})}\left[\sum_{n=1}^{\infty}\sum_{m=1}^{\infty} \lambda^n (\lambda^m)^*\right]
     &= \E_{r,\theta}\left[\sum_{n=1}^{\infty}\sum_{m=1}^{\infty} r^{\frac{1}{2}(n+m)} e^{i2\pi (n-m)\theta}\right]\\
     &= \sum_{n=1}^{\infty}\sum_{m=1}^{\infty} \E_{r}\left[r^{\frac{1}{2}(n+m)}\right] \E_{\theta}\left[e^{i2\pi (n-m)\theta}\right]
\end{align}
The expectation w.r.t $\theta$ is non-zero only if $n=m$, therefore
\begin{align}
    \E_{\lambda\sim\mathbb{T}(r_{\min}, r_{\max})}\left[\sum_{n=1}^{\infty}\sum_{m=1}^{\infty} \lambda^n (\lambda^m)^*\right]
     &= \sum_{n=1}^{\infty} \E_{r}\left[r^{n}\right]\\ &=\E_{r}\left[\sum_{n=1}^{\infty} r^{n}\right] \\
     &= \E_{r}\left[\frac{1}{1-r}\right]\\ &=\frac{1}{r_{\max}^2-r_{\min}^2}\int_{r_{\min}^2}^{r_{\max}^2} \frac{1}{1-r} dr\\
     &=\frac{1}{r_{\max}^2-r_{\min}^2}(-\log(|1-r_{\max}^2|)+\log(|1-r_{\min}^2|))\\
&=\frac{1}{r_{\max}^2-r_{\min}^2}\log\left(\frac{1-r_{\min}^2}{1-r_{\max}^2}\right).
 \end{align}

The \textit{white noise input} case is simpler. Let us start from $\|x_\infty\|_2^2 = \sum_{n=1}^{\infty}\sum_{m=1}^{\infty} \tilde u_{k-m}^*\left(A^m\right)^* A^n \tilde u_{k-n}$. Now, we can retrieve the single sum by the fact that $A$ is diagonal and $\E[\tilde u_{k-m}^* \tilde u_{k-n}]=0$ for $m\ne n$. The rest of the proof is identical, and presented in the main paper for the one-simensional setting.
\end{proof}

\end{document}